\documentclass[numbers,webpdf,imaiai]{ima-authoring-template}%
\usepackage{appendix}
\graphicspath{{Fig/}}

\theoremstyle{thmstyletwo}%
\newtheorem{theorem}{Theorem}
\newtheorem{proposition}[theorem]{Proposition}%
\newtheorem{remark}{Remark}%
\newtheorem{lemma}{Lemma}
\newtheorem{definition}{Definition}
\newtheorem{assumption}{Assumption}

\numberwithin{equation}{section}

\newcommand{\RR}{\mathbb{R}}
\newcommand{\Id}{\mathrm{Id}}
\newcommand{\Hi}{\mathcal{H}}
\DeclareMathOperator*{\Argmin}{Argmin}
\newcommand{\prox}{\normalfont \textrm{prox}}

\newcommand{\lp}{\normalfont \textrm{lp}}
\newcommand{\dom}{\textrm{dom }}
\DeclareMathOperator*{\argmin}{arg\,min}
\newcommand{\vertiii}[1]{{\left\vert\kern-0.25ex\left\vert\kern-0.25ex\left\vert #1
    \right\vert\kern-0.25ex\right\vert\kern-0.25ex\right\vert}}

\begin{document}

\DOI{DOI HERE}
\copyrightyear{2021}
\vol{00}
\pubyear{2021}
\access{Advance Access Publication Date: Day Month Year}
\appnotes{Paper}
\copyrightstatement{Published by Oxford University Press on behalf of the Institute of Mathematics and its Applications. All rights reserved.}
\firstpage{1}


\title[A BCD framework for non-convex composite opt. App. to precision matrix estimation]{A block-coordinate descent framework for non-convex composite optimization. Application to sparse precision matrix estimation.}

\author{Guillaume Lauga*
\address{\orgdiv{Laboratoire J.A. Dieudonné}, \orgname{Université Côte d'Azur}, \orgaddress{\street{28 Avenue de Valrose}, \postcode{06103, Nice}, \country{France}}}}

\authormark{Guillaume Lauga}

\corresp[*]{Corresponding author: \href{email:guillaume.lauga@univ-cotedazur.fr}{guillaume.lauga@univ-cotedazur.fr}}

\received{Date}{0}{Year}
\revised{Date}{0}{Year}
\accepted{Date}{0}{Year}


\abstract{Block-coordinate descent (BCD) is the method of choice to solve numerous large scale optimization problems, however their theoretical study for non-convex optimization, has received less attention. In this paper, we present a new block-coordinate descent (BCD) framework to tackle non-convex composite optimization problems, ensuring decrease of the objective function and convergence to a solution. This framework is general enough to include variable metric proximal gradient updates, proximal Newton updates, and alternated minimization updates. This generality allows to encompass three versions of the most used solvers in the sparse precision matrix estimation problem, deemed Graphical Lasso: graphical ISTA \citep{rolfs2012iterative}, Primal GLasso \citep{mazumder2012}, and QUIC \citep{hsieh2014quic}. We demonstrate the value of this new framework on non-convex sparse precision matrix estimation problems, providing convergence guarantees and up to a $100$-fold reduction in the number of iterations required to reach state-of-the-art estimation quality.}


\maketitle

\section{Introduction}
We introduce in this paper a new block coordinate descent framework to solve the following separable composite optimization problem
\begin{equation} \label{eq:optim_general}
    \mathbf{\widehat{\mathbf{x}}} \in \Argmin_{\mathbf{x}=(x_1,\dots,x_{L}) \in \Hi} \Psi(\mathbf{x}) = f(\mathbf{x}) + g( \mathbf x)
\end{equation}
where $\Hi$ is the direct sum of real, separable, and finite dimensional Hilbert spaces
$(\Hi_\ell)_{1\le \ell\le L}$, 
$f:\Hi\to (-\infty,+\infty]$ is continuously differentiable, and
\begin{equation} \label{eq:composite_penalty}
    g(\mathbf x) := \sum_{\ell=1}^L (\phi_\ell \circ \psi_\ell)(x_\ell)
\end{equation}
is the sum of composite functions defined for all $\ell \in \{1,\ldots,L\}$ as $\psi_\ell: \Hi_\ell  \to [0,+\infty]$, $\phi_\ell : [0,+\infty] \to (-\infty,+\infty]$.
This problem is both non-smooth and non-convex, and encapsulates a lot of interesting problems for instance in image restoration \citep{repetti2020forward,lauga2024multilevel}, and in sparse matrix estimation \citep{sun2018graphical,williams2020beyond,pouliquen2025quete}. 
Indeed, this class of composite functions include widely popular non-convex, sparsity inducing penalties such as SCAD \citep{fan2001variable}, MCP \citep{zhang2010nearly}, $\ell_q^q$ regularizers with $q\in(0,1)$ \citep{grasmair2008sparse} or log-sum penalty \citep{candes2008reweighted,prater2022proximity}. 

They key strategy to solve this problem is to iteratively minimize approximations of the objective function $\Psi$, notably avoiding to computing the proximal operator of $g$. The optimization problem is decomposed into smaller subproblems that are easier to solve \citep{nutini2022let,nutini2015coordinate,salzo_parallel_2022,repetti2021variable}. In this paper, we construct a \textit{block-coordinate descent} algorithm exploiting a similar decomposition of Problem \eqref{eq:optim_general}.  
\paragraph{Composite optimization.} 
Examples of algorithms able to solve variations of Problem \eqref{eq:optim_general} are numerous \citep{burke1985descent,cartis2011composite,fletcher2009model,drusvyatskiy2016error,lewis2015proximal,powell1983general,powell1984global,wright1990convergence,chen1997convergence,combettes2005forward,combettes2021fixed,chouzenoux2014variable}. 
Here, we focus on those that can solve our problem with composite penalties (see \eqref{eq:composite_penalty}). 
Assuming that $f$ is continuously differentiable, and that the proximity operator of $g$ is available, a common approach is the forward-backward algorithm \citep{chen1997convergence,combettes2005forward,combettes2021fixed,chouzenoux2014variable} whose convergence to a critical point has been established under convexity \citep{combettes2005forward,combettes2021fixed}, and  under non-convexity \citep{attouch2013convergence} assumptions, and which can be accelerated via preconditioning strategies \citep{chouzenoux2014variable,chouzenoux2016block,frankel2015splitting}. 
These methods, known as variable metric forward-backward algorithms, display great performance in non-convex settings \citep{chouzenoux2014variable,repetti2014preconditioned}, and such idea can be applied to construct variable metric block-coordinate forward-backward approaches \citep{chouzenoux2016block}. %
To handle the specific structure of Problem \eqref{eq:optim_general}, with $g$ being a composite function, authors of \citep{chouzenoux2013majorize,chouzenoux2014variable,chouzenoux2016block,chen1997convergence,ochs2015iteratively,geiping2018composite,repetti2021variable} proposed to iteratively majorize the function $g$ with a functional easier to minimize. The introduction of \citep{repetti2021variable} provides an exhaustive discussion about the variations of this approach with a forward-backward (proximal-gradient) backbone. We also follow this approach to handle the structure of Problem \eqref{eq:optim_general}.

In this work, we assume that Problem \eqref{eq:optim_general} is such that for every $\ell \in \{1,\ldots,L\}$, the function $\psi_\ell:\Hi_\ell$ is convex, proper, lower semi-continuous and Lipschitz continuous on its domain. Furthermore, $\phi_\ell:[0,+\infty] \rightarrow (-\infty,+\infty]$ is a concave, strictly increasing and differentiable function, with a locally Lipschitz continuous derivative such that $(\phi'_\ell \circ \psi_\ell)$ is Lipschitz continuous on the domain of $\psi_\ell$. 
Our main challenge is the non-proximability of the composition in general. Fortunately, in a lot of interesting contexts, the proximity operator of $\psi_\ell$ is available.
This has led authors of \citep{repetti2021variable} to propose an algorithm that uses a majorant of $g$ whose proximity operator only involves the one of $\psi_\ell$. This majorant needs not to be updated at each iteration, thus providing a framework analogous to that of reweighted minimization \citep{candes2008reweighted}. Under minimal assumptions, authors of \citep{repetti2021variable} were able to construct a forward-backward algorithm that converges to a critical point of the objective function.

In this work, we build upon their results and propose to exploit the \emph{block structure} of Problem \eqref{eq:optim_general} to unlock substantial computational improvements, especially in high-dimensional regimes.

\paragraph{Block-coordinate descent.} %
BCD is a generic optimization framework that has been widely adopted across various fields of machine learning \citep{wright2015coordinate}, statistics \citep{friedman2007pathwise} or deep learning \citep{pmlr-v97-zeng19a,gratton2024block,zhang2017convergent} due to its amenability to large scale applications. BCD methods come in many flavors, and the right choice of algorithm depends on the problem at hand (as can be seen in \cite{pouliquen2025quete} for the Graphical Lasso problem). 

A first approach is to minimize $\Psi$ one block at a time $\ell$, keeping the other fixed.
This method is closely related to Gauss-Seidel method for equation solving \citep{ortega1970iterative} and some convergence guarantees can be shown under various but strict assumptions \citep{ortega1970iterative,bertsekas1999nonlinear,tseng2001convergence}. Notably, convexity of $\Psi$ along the blocks is not enough to ensure convergence \citep{auslender1976optimisation}. However, under separability of the function $g$, these assumptions can be relaxed \citep{tseng2001convergence}: if $\Psi$ is quasi-convex and hemivariate regular in each block, if the updates follow an essentially cyclic rule (i.e., each one of them is updated at least once within a given number of iterations), and if
either $\Psi$ is pseudoconvex in every pair of blocks or has at most one minimizer with respect to each block; then, every cluster point is a critical point. %

In most cases, complete minimization at each iteration is out of reach. One therefore considers (possibly inexact) descent updates instead. For the non-smooth problem at hand, convergence guarantees have been obtained for proximal-gradient descent updates in various settings in \citep{bolte2014proximal,chouzenoux2016block,richtarik2014iteration,richtarik2016parallel,fercoq2015,wright2015coordinate,salzo_parallel_2022,lauga2025BCD}. %
Depending on the update rules (i.e., the order in which blocks are updated), the convergence guarantees vary. For example, it has not been shown yet that random selection of the blocks to update at each iteration would yield convergence of the \textit{iterates} to critical points of non-convex functions \citep{fercoq2015,wright2015coordinate,richtarik2016parallel,salzo_parallel_2022}. Thus we restrict our attention to "deterministic" approaches, which can still incorporate some randomness through re-shuffling of the order of updates \cite{lauga2025BCD}. In this setting, convergence to a critical point of non-convex functions was proven for instance in \citep{chouzenoux2016block} using an essentially cyclic update rule. It was later shown in \citep{lauga2025BCD} that several blocks could be updated at the same time with the same convergence guarantees. Variable metric was incorporated in several works \citep{bonettini2018block,chouzenoux2016block}, without the possibility of parallel updates like in \cite{lauga2025BCD}, which was, until recently, a small blind spot of the analysis of deterministic block-coordinate descent approaches. %

\paragraph{Sparse precision matrix estimation.}
To motivate our study, we address  the problem of estimating sparse precision matrix. 
Estimating sparse precision matrices (inverses of covariance matrices) is an important research topic in graphical models, as they can be interpreted as the adjacency matrices of conditional independence relational graphs under the assumption of Gaussian data \citep{koller2009probabilistic}.
The Graphical Lasso \citep{banerjee2008model,friedman2008} is arguably the most popular estimator for sparse precision matrices and reads as a composite optimization problem of the form \ref{eq:optim_general}.
Consider a dataset of $P$ samples $\{X^{(i)}\}_{i=1}^P$, drawn \textit{i.i.d.} from a $d$ dimensional, zero-mean, Gaussian distribution with covariance matrix $\Sigma \in \mathbb{S}^d_{++}$ (the space of $d \times d$ symmetric, positive-definite matrices). %
The Graphical Lasso estimation problem reads
\begin{equation} \label{pb:graphical_lasso}
    \widehat{\boldsymbol{\Theta}} = \argmin_{\boldsymbol{\Theta} \in \mathbb{S}^d_{++}} - \log \det(\boldsymbol{\Theta}) + \mathrm{Tr}(S^\top\boldsymbol{\Theta}) + \lambda \sum_{i,j=1}^d |[\boldsymbol{\Theta}]^{(i,j)}|.
\end{equation}

Several approaches have been developed over the years to solve this problem \citep{banerjee2008model,friedman2008,mazumder2012,pouliquen2024schur} (see \citep{pouliquen2025quete} and references therein for a computational comparison). Due to the large scale of the problem, most of these methods are BC approaches, except Graphical-ISTA \citep{rolfs2012iterative}, which computes full updates. 
The BC methods can be divided into two classes: alternated minimization (e.g. Primal G-Lasso \citep{mazumder2012}), and variable metric proximal gradient or proximal Newton methods (e.g. QUIC \citep{hsieh2013big,hsieh2014quic} or \citep{treister2014block}). They are designed to solve the \emph{convex} Graphical Lasso problem. However, the resulting estimator is biased due to the $\ell_1$ penalization as this penalty shrinks \emph{all} non-zero coefficients to zero \citep{beck2009fast}. 

Several improvements can be made in order to improve the quality of the estimated precision matrix $\boldsymbol{\widehat{\Theta}}$, the most significant of which is to replace the $\ell_1$-norm by a \emph{non-convex} surrogate.
Non-convex penalties are well known to mitigate the bias of the $\ell_1$-norm \citep{fan2001variable,bertrand2022beyond,selesnick2014sparse,fan2018lamm}. Interestingly, the solvers that were designed to tackle the convex Graphical Lasso problem can be easily used to solve non-convex ones, with iterative $\ell_1$-reweighting \citep{candes2008reweighted,pouliquen2025quete,sun2018graphical}.
However, this comes at the price of a much more important computational cost. Indeed, instead of solving one minimization problem, a sequence of the same minimization problem is solved (between $4$ \citep{sun2018graphical,candes2008reweighted} and $20$ \citep{pouliquen2025quete}, in practice).

\paragraph{Contributions.}  Our contributions are two-fold:
\begin{enumerate}
    \item Inspired by \citep{repetti2021variable}, we propose a new, and general, variable metric block-coordinate descent framework for the minimization of non-smooth, non-convex composite optimization problems. We show under standard assumptions that with this framework, \emph{convergence to a critical point of the underlying optimization problem \eqref{eq:optim_general} is guaranteed.} The generality of the framework allows us to propose (i) a variable metric block-coordinate forward-backward algorithm (amenable to some inexactness), (ii) a block-coordinate proximal Newton algorithm, and (iii) a Gauss-Seidel-like algorithm. A notable difference with previous works is the possibility of deterministic parallel updates.%
    \item With this general framework, we can encompass several algorithms designed to tackle the convex Graphical Lasso. Hence, we can avoid the full minimization previously required at each iteration of the reweighting algorithm \citep{sun2018graphical,pouliquen2025quete}, which greatly accelerates the convergence of the algorithms in practice. In particular, we can cast with minor modifications the following algorithms in the non-convex setting:  Graphical ISTA \cite{rolfs2012iterative}, P(rimal) G-Lasso \citep{mazumder2012}, and QUIC \citep{hsieh2013big,hsieh2014quic}. Numerical experiments show that estimation performances can be maintained for a fraction of the previous computational cost. %
\end{enumerate}

\paragraph{Organization of the paper.}
This paper is organized as follows. In Section \ref{sec:background} we present the usual concepts and definitions in non-convex optimization. We then introduce in Section \ref{sec:framework} our proposed framework and its convergence analysis. This analysis is first tailored for a variable metric and inexact block-coordinate forward-backward algorithm and then adapted to its Newton and Gauss-Seidel version. We conclude this paper with numerical experiments on the non-convex Graphical Lasso problem in Section \ref{sec:graphical_LASSO}.

\paragraph{Notations.} 
We denote by $\vertiii{\cdot}$ the Euclidean norm on $\Hi$ and by $\Vert \cdot \Vert$ the Euclidean norm on the $L$ spaces $(\Hi_\ell)_{1\leq \ell \leq L}$. Similarly, the scalar product on $\Hi$ will be denoted by $\langle \langle \cdot,\cdot \rangle \rangle$ and the scalar product on $\Hi_\ell$ by  $\langle \cdot,\cdot \rangle$; the potential ambiguity between two spaces is cleared up as the variables on which the scalar product is applied will be indexed by $\ell$. Note that for all $\mathbf{x},\mathbf{y} \in \Hi$, $\langle \langle \mathbf{x},\mathbf{y} \rangle \rangle = \sum_{\ell=1}^L \langle x_\ell,y_\ell \rangle$. For a closed set $\Omega$, the distance function to $\Omega$ is defined for all $\mathbf{x}\in\Hi$ as $\mathrm{dist}_\Omega(\mathbf{x}) = \inf_{\mathbf{y} \in \Omega} \vertiii{\mathbf{x}-\mathbf{y}}$. To refer to the functional taken with respect to a set of blocks indexed by $J \subseteq \{1,\ldots,L\}$ while the other blocks are fixed, we will note it : $(x_{\ell \in J}\mid x_{\ell \notin J})$. If $J$ only index is $\ell$, we will simply write $(x_\ell \mid x_{\bar \ell})$. Hence,
\begin{equation*}
    \forall y \in \Hi_\ell, \quad f(y,x_{\bar \ell}) = f(x_1,x_2, \ldots, x_{\ell-1},y,x_{\ell+1},\ldots,x_L).
\end{equation*}
For a continuously differentiable function $f$, we denote for all $\mathbf{x} \in \Hi$ by $\nabla_\ell f(\mathbf{x})$ the gradient of $f$ taken with respect to the variables in the $\ell$-th block, and $\nabla_J f(\mathbf{x})$ if it is taken with respect to a subset $J$ of variables. Similarly, $\partial_J f$ will denote the $\ell \in J$-components of the set $\partial f$.

\section{Optimization background} \label{sec:background}
In this section, we present definitions of concepts and results we use throughout this paper in a non-convex setting. For additional information on non-convex optimization, we refer the reader to \citep{VarAnalRockafellar}.
\subsection{Subgradients}
\begin{definition}{\textbf{Subgradient \citep{VarAnalRockafellar}.}}
    Let $\Psi:\Hi \mapsto (-\infty+\infty]$, and let $\mathbf{x}\in\Hi$. The Fréchet subdifferential of $\Psi$ at $\mathbf{x}$ is denoted by $\hat{\partial} \Psi(\mathbf{x})$ and is given by the set
    \begin{align*}
        \hat{\partial} \Psi(\mathbf{x}) = \left\{ \mathbf{\hat{s}} \in \Hi \;|\; \lim_{\mathbf{y}\rightarrow \mathbf{x}} \inf_{\mathbf{y} \neq \mathbf{x}} \frac{1}{\vertiii{\mathbf{x}-\mathbf{y}}} \left(\Psi(\mathbf{y})-\Psi(\mathbf{x})-\langle \mathbf{y}-\mathbf{x}, \mathbf{\hat{s}} \rangle \right) \geq 0 \right\}.
    \end{align*}
    If $\mathbf{x} \notin \dom \Psi$, then $\hat{\partial} \Psi(\mathbf{x}) = \emptyset$.
    The limiting subdifferential of $\Psi$ at $\mathbf{x}$ is denoted by $\partial \Psi(\mathbf{x})$ and is given by
    \begin{align*}
\partial \Psi(\mathbf{x}) = \big\{ \mathbf{\hat{s}} \in \Hi \;|\; \exists & \left( \mathbf{x}^k,\mathbf{\hat{s}}^k \right) \overset{k\rightarrow\infty}{\rightarrow} \left(\mathbf{x},\mathbf{\hat{s}}\right) \\ & \textrm{ such that } \Psi(\mathbf{x}^k) \overset{k\rightarrow\infty}{\rightarrow} \Psi(\mathbf{x}) \textrm{ and } (\forall k \in \mathbb{N}) ~\mathbf{\hat{s}}^k \in \hat{\partial} \Psi(\mathbf{x}^k) \big\}.
\end{align*}
Both $\hat{\partial} \Psi (\mathbf{x})$ and $\partial \Psi(\mathbf{x})$ are closed sets \citep[Theorem 8.6]{VarAnalRockafellar}.
Recall that if $\Psi$ is convex, its subdifferential is given for all $\mathbf{x} \in \Hi$ by
\begin{equation}
\label{eq:subdif_convex}
    \partial \Psi(\mathbf{x}) = \{\mathbf{s} \in \Hi \;|\; \Psi(\mathbf{x}) + \langle \mathbf{s}, \mathbf{y}-\mathbf{x} \rangle \leq \Psi(\mathbf{y}), \forall \mathbf{y} \in \Hi \}.
\end{equation}

\end{definition}

This subdifferential has notable properties. If $\mathbf{x}$ is a (local) minimizer of $\Psi$ if $\mathbf{0}\in \partial \Psi(\mathbf{x})$.
\begin{proposition}{\textbf{Chain rule \citep{VarAnalRockafellar}}.}
Let $f:\Hi \rightarrow \RR$ be a differentiable function and $g:\Hi \rightarrow \RR$, then we have $\partial(f+g) = \nabla f + \partial g$.
\end{proposition}
We also have the following separability result for the subdifferential:
\begin{proposition}\label{prop:separability}{\textbf{Separability of the subdifferential \citep{VarAnalRockafellar}.}} Suppose that $f$ is continuously differentiable and that $g$ is separable in $\Hi = \Hi_1 \times \ldots \times\Hi_L$. Then for all $\mathbf{x} = (x_1,\ldots,x_L)\in \Hi_1 \times \ldots \times\Hi_L$ we have
    \begin{equation*}
    \partial \Psi(\mathbf{x}) = (\nabla_1 f(\mathbf{x}) + \partial(\phi_1 \circ \psi_1)(x_1)) \times \ldots \times (\nabla_L f(\mathbf{x}) +\partial(\phi_L \circ \psi_L)(x_L)).
    \end{equation*}
    \end{proposition}
The last property we need was proven in \citep{repetti2021variable} and allows us to express the subdifferential of $(\phi \circ \psi)$ in terms of the subdifferential of $\psi$.
\begin{proposition}{\citep[Proposition 2.10]{repetti2021variable}} \label{prop:chain_rule}
Let $\psi:\Hi \rightarrow [0,+\infty]$ be a proper function, continuous on its domain, and let $\phi:[0,+\infty] \rightarrow (-\infty,+\infty]$ be a concave, strictly increasing and differentiable function. We further assume that $(\phi' \circ \psi)$ is continuous on its domain. Then, for every $\mathbf{y} \in \Hi$, we have $\partial(\phi \circ \psi)(\mathbf{y}) = (\phi' \circ \psi(\mathbf{y}) )\partial \psi(\mathbf{y})$.
\end{proposition}
\paragraph{The Kurdyka-Łojasiewicz (KŁ) property.}
A specific class of concave and continuous functions, called desingularizing functions, are of at the heart of the KŁ framework to handle non-convexity. 
\begin{definition}{\textbf{Concave and continuous functions \citep{bolte2014proximal}.}} \label{def:desingularizing} Let $\eta \in (0,+\infty]$. We denote by $\Phi_\eta$ the class of all concave and continuous functions $\varphi : [0,\eta)\rightarrow\RR_+$ that satisfy the following conditions:
\begin{enumerate}
    \item $\varphi(0)=0$,
    \item $\varphi$ is $C^1$ on $(0,\eta)$ and continuous at $0$,
    \item for all $s \in (0,\eta)$, $\varphi'(s)>0$.
\end{enumerate}
\end{definition}
\noindent This desingularizing function is then used to characterize how sharp can the function $\Psi$ be if it satisfies the Kurdyka-Łojasiewicz property.
\begin{definition}{\textbf{Kurdyka-Łojasiewicz property \citep{bolte2014proximal}.}} \label{def:kl}
    Let $\Psi : \Hi \rightarrow (-\infty, + \infty]$ be proper and lower semicontinuous.
    \begin{enumerate}
        \item The function $\Psi$ is said to have the \textit{KŁ property} at $\mathbf{\tilde{u}}\in \dom \partial\Psi$ if there exist $\eta \in (0,+\infty]$, a neighborhood $U$ of $\mathbf{\tilde{u}}$ and a function $\varphi \in \Phi_\eta$  such that for all 
        \begin{equation*}
            \mathbf u \in U \cap [\Psi(\mathbf{\tilde{u}}) < \Psi(\mathbf{u}) < \Psi(\mathbf{\tilde{u}}) + \eta],
        \end{equation*}
        the following inequality holds
        \begin{equation*}
            \varphi'(\Psi(\mathbf{u}) - \Psi(\mathbf{\tilde{u}})) \mathrm{dist} (0,\partial \Psi(\mathbf{u})) \geq 1.
        \end{equation*}
        Recall that $\mathrm{dist}(\mathbf{x},\partial \psi (\mathbf{u})) = \inf_{\mathbf{s} \in \partial \psi(\mathbf{u})} \vertiii{\mathbf s- \mathbf{x}}$.
        \item If $\Psi$ satisfies the KŁ property at each point of $\dom \partial\Psi$, then $\Psi$ is called a KŁ function.
    \end{enumerate}
\end{definition}

The following lemma presents the KŁ property in a practical form by unifying the notion of neighborhood across its level curves. 
\begin{lemma}{\textbf{Uniformized KŁ property \citep{bolte2014proximal}.}} \label{lm:unif_KL}Let $\Omega$ be a compact subset of $\Hi$. Let $\Psi:\Hi \to (-\infty,+\infty]$ be a proper and lower semicontinuous function, constant on $\Omega$ and satisfying the KŁ inequality on $\Omega$. Then there exists $\epsilon>0, \eta>0$, and $\varphi\in\Phi_\eta$ such that for all $\tilde{u}\in \Omega$ and all $u$ satisfying
\begin{equation*}
    \mathrm{dist} (\mathbf u,\Omega) < \epsilon \text{ and } [\Psi(\mathbf{\tilde{u}}) < \Psi(\mathbf{u}) < \Psi(\mathbf{\tilde{u}}) + \eta]
\end{equation*}
one has
\begin{equation}
    \varphi'(\Psi(\mathbf{u}) - \Psi(\mathbf{\tilde{u}})) \mathrm{dist} (0, \partial \Psi(\mathbf{u})) \geq 1. \label{eq5:inequality_KL}
\end{equation}
\end{lemma}
\begin{remark}
The KŁ property is satisfied by a large class of functions, e.g., all those that can be expressed with $\log-\exp$, or semi-algebraic functions. We refer the reader to \citep{attouch2009convergence,attouch2010proximal,attouch2013convergence,bolte2010characterizations,bolte2014proximal} for a more detailed presentation of such functions.
\end{remark}
\subsection{Variable metric optimization}
Changing the metric used to compute the proximity operator is a common technique to accelerate the convergence of proximal algorithms in practice. The metric is obtained choosing appropriate symmetric, positive definite matrices \citep{chouzenoux2014variable,chouzenoux2016block}. We define the weighted norms $\vertiii{ \cdot }_A$ associated with a symmetric positive definite matrix $A\in\RR^{N \times N}$ as follows:
\begin{equation*}
    (\forall \mathbf{x} \in \Hi) \quad \vertiii{\mathbf{x}}_A = \langle\langle \mathbf x,A\mathbf{x}\rangle\rangle^{1/2},
\end{equation*}
and for each block $\ell\in \{1,\ldots,L\}$,
\begin{equation*}
    (\forall {x}_\ell \in \Hi_\ell) \quad \Vert x_\ell \Vert_{A_\ell} = \langle x_\ell,A_\ell x_\ell \rangle^{1/2}, 
\end{equation*}
where $A_\ell$ is the restriction of $A$ to the block $\ell$. The Loewner partial order is defined as follows on $\Hi$, and hence on $\Hi_\ell$ for all $\ell \in \{1,\ldots,L\}$:
\begin{equation*}
   A_1 \preceq A_2 \Leftrightarrow (\forall x_\ell \in \Hi_\ell), \quad \langle x_\ell, A_1 x_\ell \rangle \leq \langle x_\ell, A_2 x_\ell  \rangle.
\end{equation*}

\noindent The proximity operator relative to a metric $A$ is defined as follows:
\begin{definition}
\textbf{Proximity operator \citep{rockafellar1970convex,bauschke2017}.}
Let $g:\Hi \rightarrow (-\infty,+\infty]$ be a proper, lower semi-continuous function. Let $A: \Hi \rightarrow \Hi$ a positive definite matrix. The proximity operator of $g$ at $\mathbf{x} \in \Hi$, $\prox_g^A: \Hi \rightrightarrows \Hi$ is defined as:
\begin{equation}
    \prox_{g}^A(\mathbf{x}) := \Argmin_{\mathbf{y} \in \Hi} \left\{ g(\mathbf{y}) + \frac{1}{2} \vertiii{ \mathbf{x}-\mathbf{y}}_A^2 \right\}.
\end{equation}
\end{definition}
$\prox_g^A$ at $\mathbf{x}$ is a non-empty set whenever $g$ is proper and lower semicontinuous and bounded from below.

\section{Our block-coordinate descent framework for composite non-convex optimization.} \label{sec:framework}
In this section, we present our general variable metric block-coordinate descent framework to tackle problems of the form \eqref{eq:optim_general}. 
\subsection{Assumptions}
The assumptions on Problem \eqref{eq:optim_general} are similar as in \citep{repetti2021variable}, but we state them here in a block version, hence they differ from the presentation done in \cite{repetti2021variable}.
\begin{assumption}\label{ass:1} 
\begin{enumerate}[(ii)]
    \item The function $f$ is continuously differentiable, and there exist $(\beta_{\ell,j})_{\ell,j \in \{1,\ldots,L\}} \in \RR_{++}$ such that
        \begin{align*}
            (\forall \ell,j\in\{1,\ldots,L\}), &(\forall \mathbf{x} \in \Hi),(\forall v_j \in \Hi_j) \nonumber\\
            & \Vert \nabla_\ell f(\mathbf{x}+(0,\ldots,0,v_j,0,\ldots,0)) - \nabla_\ell f(\mathbf{x}) \Vert \leq \beta_{\ell,j} \Vert v_j \Vert
        \end{align*}
    \item For every $\ell \in \{1,\ldots,L\}$, $\psi_\ell:\Hi_\ell \rightarrow [0,+\infty]$ is a convex, proper and lower-semicontinuous function. It is also Lipschitz-continuous on its domain.
    \item For every $\ell \in \{1,\ldots,L\}$, $\phi_\ell:[0,+\infty] \rightarrow (-\infty,+\infty]$ is a concave, strictly increasing and differentiable function on $[0,+\infty[$.
    \item For every $\ell \in \{1,\ldots,L\}$, $\phi'_\ell$ is locally Lipschitz continuous on its domain.
\end{enumerate}
\end{assumption}

The block smoothness of $f$ yields a multiple block smoothness \citep{lauga2025BCD}, a common property in the stochastic BCD literature (see for instance \citep{salzo_parallel_2022}).
\begin{proposition}{Multiple block smoothness \citep[Proposition 1]{lauga2025BCD}.} \label{prop:multi_block_smoothness}
    Suppose that Assumption \ref{ass:1} holds. For all $\boldsymbol{\varepsilon} = (\varepsilon_\ell)_{1 \leq \ell \leq L} \in \{0,1\}^L$, there exists $\beta=\sqrt{\sum_{\ell,j=1}^L \varepsilon_j \beta_{\ell,j}^2}>0$ such that for all $1\leq \ell \leq L$, $v_\ell \in \Hi_\ell$ we have
        \begin{equation*}
            (\forall \mathbf{x} \in \Hi) \quad \vertiii{\nabla f(\mathbf{x}+(\varepsilon_\ell v_\ell)_{\ell\in\{1,\ldots,L\}})-\nabla f ( \mathbf{x})  } \leq \beta \vertiii{(\varepsilon_\ell v_\ell)_{\ell\in\{1,\ldots,L\}}}.
        \end{equation*}
\end{proposition}
We also need an assumption on the updates of the variables, which is a generalization of \citep[Assumption 3]{lauga2025BCD}.
\begin{assumption} \label{ass:3} 
In each outer iteration $k$, each block has been updated at least once. Formally, for all $k \in \mathbb{N}$, $i\in\{0,\ldots,I_k-1\}$ denote by $J_i^k$ the set of the blocks updated at iteration $(k,i)$: $J_i^k =\{\ell \;|\; \varepsilon_{i,\ell}^k=1\} \subseteq \{1,\ldots,L\}$. For all $j$
\begin{equation}
    \bigcup_{i=0}^{I_k-1} J_i^k = \{1,\ldots,L\}.
\end{equation}
\end{assumption}

\begin{remark}
    In typical deterministic BCD algorithms, each block needs to be updated an infinite number of times to obtain convergence in the limit, i.e., there exists a finite $K\in\mathbb{N}\backslash\{0\}$ such that for every $K$ iterations, each block has been updated at least once \citep{lauga2025BCD,chouzenoux2016block}. Therefore, the value of $I_k$ should be superior or equal to $K$. For instance in a cyclic algorithm, each block is updated only once, therefore $I_k \propto L$.
\end{remark}
\paragraph{Majorization of $f$ through variable metric.}
To accomodate the change of the metric, a quadratic majorant of $f$ is constructed at each iteration with positive definite matrices, with respect to the selected blocks to update.
\begin{assumption} \label{ass:2}
For any point $\mathbf{y}\in\Hi$, and any $J \subseteq \{1,\ldots,L\}$ we construct the following quadratic upper bound of $f$ at $\mathbf{y}$   
\begin{enumerate}[(ii)]
        \item We have
        \begin{align}
            (\forall x \in \mathbin{\scalebox{1.6}{$\times$}}_{\ell\in J}\Hi_\ell), & \\ & f(x \mid y_{\overline{J}}) \leq f(\mathbf{y}) + \sum_{\ell \in J}\left\langle x_\ell - y_{\ell}, \nabla_{\ell} f(\mathbf y) \right\rangle 
    + \frac{1}{2} \left\langle x_\ell - y_{\ell}, A_{y,\ell}(x_\ell - y_{\ell}) \nonumber\right\rangle
        \end{align}
        \item There exists $(\underline{\nu}, \overline{\nu}) \in \mathbb{R}_{++}^2$ such that for all $\mathbf y \in \Hi$, $J \subseteq \{1,\ldots,L\}$, the associated metric $(A_{y,\ell})_{\ell \in J}$ is such that $\underline{\nu} \Id \preccurlyeq A_{y,\ell} \preccurlyeq \overline{\nu} \Id$.
    \end{enumerate}
\end{assumption}
\noindent The block-coordinate variable metric can be constructed from the full metric \citep{chouzenoux2016block}. Suppose that, for every $\mathbf{x}' \in \mathrm{dom}\, g$. A quadratic majorant function of 
  $f$ at $\mathbf{x}'$ is given by
  \begin{equation}
  (\forall {x} \in \mathbb{R}^N) \quad 
  Q(\mathbf{x} \mid \mathbf{x}') := f(\mathbf{x}') + \left\langle \mathbf{x} - \mathbf{x}', \nabla f(\mathbf{x}') \right\rangle 
  + \frac{1}{2} \left\langle \mathbf{x} - \mathbf{x}', B(\mathbf{x}')( \mathbf{x} - \mathbf{x}') \right\rangle, \label{eq:construct_variable_metric}
  \end{equation}
  where $B(\mathbf{x}') \in \mathbb{R}^{N \times N}$ is a symmetric positive definite matrix. Then, Assumption \ref{ass:2} is satisfied for $A_{\ell} = 
  \left(B(\mathbf{x})_{(n, n')}\right)_{(n, n') \in J_{\ell}^2}$, where, for every 
  $(n, n') \in \{1, \ldots, N\}^2$, $B(\mathbf{x})_{(n, n')}$ denotes the $(n, n')$ element of matrix 
  $B(\mathbf{x})$. Moreover, if there exists $(\nu,  \overline{\nu}) \in (0, +\infty)^2$ such that, for every 
  ${x}' \in \mathrm{dom}\, g$, $\nu I_N \preceq B(\mathbf{x}') \preceq  \overline{\nu} I_N$, then 2. is satisfied.

If $\mathrm{dom}\, g$ is convex, the existence of the majorant function is ensured when $f$ has a Lipschitz continuous gradient 
  (see \citep[Lem.\ 3.1]{chouzenoux2014variable}).
\begin{remark}
    This assumption is satisfied immediately as a consequence of the descent Lemma \citep{bertsekas1999nonlinear,ortega2000iterative} if for every $k \in \mathbb{N}$, $i \in \{0,\ldots,I_k\}$, $\ell \in \{1,\ldots,L\}$, $A_{i,\ell}^k$ is chosen as $\beta_\ell \Id$.
\end{remark}
We also need to ensure that $I_k$ will remain finite for all $k$.
\begin{assumption} \label{ass:4}
   There exists $\bar{I} \in \mathbb{N}^*$, such that for every $k \in \mathbb{N}$, $0<I_k \leq \bar{I} <+\infty$.
\end{assumption}
And that the iterates produced by the algorithm are bounded. 
\begin{assumption} \label{ass:bounded}
    $\Psi$ is bounded from below, and the iterates generated by the algorithm are bounded.
\end{assumption}
This assumption is satisfied whenever $\Psi$ is coercive, i.e., $\lim_{\Vert \mathbf{x} \Vert \rightarrow +\infty} \Psi(\mathbf{x}) = +\infty$. For the application we have in mind, $\Psi$ is unfortunately not coercive, hence this assumption which is reasonable in practice.
\paragraph{Majorization of the function $g$} 
We conclude the assumptions by defining the majorant of the function $g$. 
With Assumption \ref{ass:1}, the function $g$ can be upper bounded by a majorant $q^g(\cdot, \mathbf{y})$ defined at point $\mathbf{y}\in\Hi$:
\begin{equation}
    (\forall \mathbf{x} \in \Hi) \quad \left\{\begin{array}{ll}
        g(\mathbf{x}) & \leq q^g(\mathbf{x}, \mathbf{y})  = \sum_{\ell=1}^L q^g_\ell(x_\ell,y_\ell),\\
       g(\mathbf{y}) & = q^g(\mathbf{y},\mathbf{y}).
    \end{array}\right.
\end{equation}
such that, for every $\ell\in \{1,\dots,L\}$, $q^g_\ell(\cdot,y_\ell):\Hi \rightarrow (-\infty,+\infty]$ is a majorant of $(\phi_\ell \circ \psi_\ell)$ at $y_\ell$, i.e.,
\begin{equation}
    (\forall \mathbf{x} \in \Hi), \quad \left\{\begin{array}{ll}

        (\phi_\ell \circ \psi_\ell)(x_\ell) \leq q^g_\ell(x_\ell, y_\ell) \\
        (\phi_\ell \circ \psi_\ell)(y_\ell) = q^g_\ell(y_\ell,y_\ell)
    \end{array}\right.
\end{equation}
This majorant is obtained by taking for every $\ell\in \{1,\ldots,L\}$, the tangent of the concave differentiable function $\phi_\ell$ at $\psi_\ell(y_\ell)$:
\begin{equation}
    (\forall x_\ell \in \Hi_\ell), \quad q^g_\ell(x_\ell,y_\ell) = (\phi_\ell \circ \psi_\ell)(y_\ell) + (\phi_\ell' \circ \psi_\ell)(y_\ell)(\psi_\ell(x_\ell) - \psi_\ell(y_\ell)).
\end{equation}
Two terms of the majorant are constant w.r.t. $x_\ell$, hence $q^g$ can be rewritten as:
\begin{align*}
    q^g(\cdot, \mathbf{y}) 
    & = \sum_{\ell=1}^L (\phi_\ell \circ \psi_\ell)(y_\ell) + (\phi_\ell' \circ \psi_\ell)(y_\ell)(\psi_\ell(\cdot) - \psi_\ell(y_\ell)), \\
    & = \sum_{\ell=1}^L (\phi_\ell \circ \psi_\ell)(y_\ell) - (\phi_\ell' \circ \psi_\ell)(y_\ell)\psi_\ell(y_\ell) + \lambda_{\ell,y} \psi_\ell(x_\ell), \\
    & = C + \sum_{\ell=1}^L \lambda_{\ell,y} \psi_\ell(x_\ell),
\end{align*}
where for all $\ell\in \{1,\ldots,L\}$, $\lambda_{\ell,y} = (\phi_\ell' \circ \psi_\ell)(y_\ell)$. $C$ is a constant, and therefore will not affect the computation of the proximity operators of $q^g(\cdot, \mathbf{y})$. Therefore at a given iteration $k$ of our algorithms, this computation will be replaced by the computation of the proximal operator w.r.t. $m_k := \sum_{\ell=1}^L \lambda_{\ell,k} \psi_\ell(\cdot)$ in the algorithm, which is constructed from the majorant at the point $\mathbf{x}^k$, the $k-$th iterate of the algorithm.

Finally, the assumption that will unlock the convergence to a critical point of $\Psi$ is the Kurdyka-Łojasiewicz property. 
\begin{assumption}\label{ass:5}
    The function $\Psi:=f+g$ satisfies the KŁ property for all $\mathbf{x} \in \dom \Psi$.
\end{assumption}
\paragraph{A block-coordinate descent algorithm for composite optimization.}
The proximal operator relative to an arbitrary metric may not be available explicitly due to the metric or the function itself \citep{salzo2012inexact}. Hence, we follow \citep{repetti2021variable} and propose an inexact version of our algorithm in Algorithm \ref{alg:MAJ-BC-FB}, named CO-BC-FB for Composite Optimization Block-Coordinate Forward-Backward. As a reminder, $k$ denotes the outer iterations, $i$ the inner iterations (during which $\Psi_k$ is fixed), and $\ell$ the block.
\begin{algorithm}[t]
\caption{\textbf{CO-BC-FB}}
\label{alg:MAJ-BC-FB}
\begin{algorithmic}[1]
\State $(\mathbf{\varepsilon}^n)_{n\in \mathbb{N}} = (\varepsilon_1^{n},\ldots,\varepsilon_L^{n})_{n\in \mathbb{N}}$ be a sequence of variables with value in $\{0,1\}^L$. 
\State $\delta \in \RR_{++}$ 
\State $\mathbf{x}^0 = (x_1^{0},\ldots,x_L^{0}) \in \dom g$.
\For{$k = 0,1,\dots$}
  \State $\mathbf{\tilde{x}}_0^k \gets \mathbf{x}^k$
  \For{$i = 0,\ldots, I_k-1$}
    \ForAll{$\ell \in J_i^k$}
      \State \textbf{Find} $\tilde{x}_{i+1,\ell}^k$, $\alpha_i\geq 1/2+\delta$, and $\tilde{v}^k_\ell(\tilde{x}_{i+1,\ell}^k)\in \partial\psi_\ell(\tilde{x}_{i+1,\ell}^k)$ such that
      \Statex \hspace{\algorithmicindent}%
      $\lambda_{\ell,k}\psi_\ell(\tilde{x}_{i+1,\ell}^k)
      + \langle \nabla_\ell f(\mathbf{\tilde{x}}_i^k),\, \tilde{x}_{i+1,\ell}^k - \tilde{x}_{i,\ell}^k \rangle
      + \alpha_i \|\tilde{x}_{i+1,\ell}^k - \tilde{x}_{i,\ell}^k\|_{A^k_{i,\ell}}^2
      \le \lambda_{\ell,k}\psi_\ell(\tilde{x}_{i,\ell}^k)$
      \Statex \hspace{\algorithmicindent}%
      $\|\nabla_\ell f(\mathbf{\tilde{x}}_i^k) + \tilde{v}^k_\ell(\tilde{x}_{i+1,\ell}^k)\|
      \le \mu \|\tilde{x}_{i+1,\ell}^k - \tilde{x}_{i,\ell}^k\|_{A^k_{i,\ell}}$
    \EndFor
  \EndFor
  \State $\mathbf{x}^{k+1} \gets (\tilde{x}_{I_k,1}^{k},\ldots,\tilde{x}_{I_k,L}^{k})$
\EndFor
\end{algorithmic}
\end{algorithm}
\subsection{Convergence analysis}
In this section, we analyze the convergence properties of the proposed algorithm. We will first show its descent properties, then prove the convergence of the iterates to a critical point of the objective function.
\paragraph{Descent properties of the algorithm} \label{subsec:sufficient_decrease_gen}
The descent behavior of our algorithm depends on the inner loop of $I_k$ iterations. Therefore we first show descent properties on $(\Psi(\mathbf{\tilde{x}}_{i}^k))_{0\leq i \leq I_k}$. Denote by $\Psi_k$ the majorant function at iteration $k$, defined by
\begin{equation}
    \Psi_k(\mathbf{x}) = f(\mathbf{x}) + q^g(\mathbf{x}, \mathbf{x}^k),
\end{equation}
and recall that $m_k = \sum_{\ell=1}^L \lambda_{\ell,k} \psi_\ell$.
\begin{proposition}{\textbf{Sufficient decrease condition on the majorant function.}}
\label{prop:descent_cycle_majorant}
Let $(\mathbf{x}^k)_{k\in\mathbb{N}}$ and, for every $k \in \mathbb{N}$, $(\mathbf{\tilde{x}}_i^k)_{0\leq i \leq I_k}$ be the iterates generated by Algorithm \ref{alg:MAJ-BC-FB}. Let $k \in \mathbb{N}$ and $0 \leq i_1, < i_2 \leq I_k$. Under Assumptions \ref{ass:1}, \ref{ass:3}, \ref{ass:2}, and \ref{ass:4}, there exists $\eta>0$ such that
\begin{equation} \label{eq:descent_cycle_inter}
    f(\mathbf{\tilde{x}}_{i_2}^k)  + m_k(\mathbf{\tilde{x}}^{k}_{i_2}) + \eta \sum_{i=i_1}^{i_2-1} \sum_{\ell \in J_i^k} \Vert \tilde{x}_{i+1,\ell}^k - \tilde{x}_{i,\ell}^k \Vert^2  \leq f(\mathbf{\tilde{x}}_{i_1}^k) + m_k(\mathbf{\tilde{x}}^{k}_{i_1})
\end{equation}
hence
\begin{equation} \label{eq:descent_cycle_majorant}
    \Psi_k(\mathbf{\tilde{x}}_{i_2}^k) + \eta \sum_{i=i_1}^{i_2-1} \sum_{\ell \in J_i^k} \Vert \tilde{x}_{i+1,\ell}^k - \tilde{x}_{i,\ell}^k \Vert^2  \leq \Psi_k(\mathbf{\tilde{x}}_{i_1}^k)
\end{equation}
\end{proposition}
\begin{proof}
    We have for every $k \in \mathbb{N}$, $i \in \{0,\ldots,I_k\}$, and $\ell \in J_i^k$ that
\begin{equation*}
    \lambda_{\ell,k} \psi_\ell (\tilde{x}^{k}_{i+1,\ell}) + \langle \nabla_\ell f(\mathbf{\tilde{x}}_i^k), \tilde{x}_{i+1,\ell}^k - \tilde{x}_{i,\ell}^k \rangle + \alpha_i \Vert \tilde{x}_{i+1,\ell}^k - \tilde{x}_{i,\ell}^k \Vert_{A_{i,\ell}^k}^2 \leq \lambda_{\ell,k} \psi_\ell (\tilde{x}_{i,\ell}^k),
\end{equation*}
which we can sum up from $\ell=1$ to $\ell = L$ using $\tilde{x}^{k}_{i+1,\ell} = \tilde{x}^{k}_{i,\ell}$, so that
\begin{equation*}
    \left(m_k(\mathbf{\tilde{x}}^{k}_{i+1}) - m_k(\mathbf{\tilde{x}}_{i}^k)\right) + \sum_{\ell \in J_i^k} \alpha_i \Vert \tilde{x}_{i+1,\ell}^k - \tilde{x}_{i,\ell}^k \Vert_{A_{i,\ell}^k}^2 \leq -\sum_{\ell \in J_i^k}  \langle \nabla_\ell f(\mathbf{\tilde{x}}_i^k), \tilde{x}_{i+1,\ell}^k - \tilde{x}_{i,\ell}^k \rangle,
\end{equation*}
We can now sum up the inequalities from $i_1$ to $i_2$:
\begin{align*}
    m_k(\mathbf{\tilde{x}}^{k}_{i_2}) - m_k(\mathbf{\tilde{x}}^{k}_{i_1}) + \sum_{i=i_1}^{i_2-1} \alpha_i\sum_{\ell \in J_i^k} \Vert \tilde{x}_{i+1,\ell}^k - \tilde{x}_{i,\ell}^k \Vert_{A_{i,\ell}^k}^2 & \leq -\sum_{i=i_1}^{i_2-1} \sum_{\ell \in J_i^k} \langle \nabla_\ell f(\mathbf{\tilde{x}}_i^k), \tilde{x}_{i+1,\ell}^k - \tilde{x}_{i,\ell}^k \rangle. \nonumber \\
    & \leq -\sum_{i=i_1}^{i_2-1} \langle \langle \nabla_\ell f(\mathbf{\tilde{x}}_i^k), \mathbf{\tilde{x}}_{i+1}^k - \mathbf{\tilde{x}}_{i}^k \rangle \rangle.
\end{align*}
We can now use Assumption \ref{ass:2} to obtain:
\begin{align*}
    -\sum_{i=i_1}^{i_2-1} \langle \langle \nabla_\ell f(\mathbf{\tilde{x}}_i^k), \mathbf{\tilde{x}}_{i+1}^k - \mathbf{\tilde{x}}_{i}^k \rangle \rangle & \leq \sum_{i=i_1}^{i_2-1} f(\mathbf{\tilde{x}}_{i}^k) - f(\mathbf{\tilde{x}}_{i+1}^k)  + \frac{1}{2} \vertiii{\mathbf{\tilde{x}}_{i+1}^k - \mathbf{\tilde{x}}_{i}^k}_{A_i^k}^2 \nonumber \\
    & \leq f(\mathbf{\tilde{x}}_{i_1}^k) - f(\mathbf{\tilde{x}}_{i_2}^k) + \sum_{i=i_1}^{i_2-1} \frac{1}{2} \vertiii{\mathbf{\tilde{x}}_{i+1}^k - \mathbf{\tilde{x}}_{i}^k}_{A_i^k}^2 \nonumber \\
    & = f(\mathbf{\tilde{x}}_{i_1}^k) - f(\mathbf{\tilde{x}}_{i_2}^k) + \sum_{i=i_1}^{i_2-1} \frac{1}{2} \sum_{\ell \in J_i^k} \Vert \tilde{x}_{i+1,\ell}^k - \tilde{x}_{i,\ell}^k \Vert_{A_{i,\ell}^k}^2,
\end{align*}
We then obtain
\begin{align*}
    m_k(\mathbf{\tilde{x}}^{k}_{i_2}) - m_k(\mathbf{\tilde{x}}^{k}_{i_1}) +  \sum_{i=i_1}^{i_2-1} \alpha_i \sum_{\ell \in J_i^k} \Vert \tilde{x}_{i+1,\ell}^k - \tilde{x}_{i,\ell}^k \Vert_{A_{i,\ell}^k}^2 & \leq f(\mathbf{\tilde{x}}_{i_1}^k) - f(\mathbf{\tilde{x}}_{i_2}^k) + \sum_{i=i_1}^{i_2-1} \frac{1}{2} \sum_{\ell \in J_i^k} \Vert \tilde{x}_{i+1,\ell}^k - \tilde{x}_{i,\ell}^k \Vert_{A_{i,\ell}^k}^2,
\end{align*}
and taking $\eta = \underline{\nu}\delta$, we obtain
\begin{equation*}
    f(\mathbf{\tilde{x}}_{i_2}^k)  + m_k(\mathbf{\tilde{x}}^{k}_{i_2}) + \eta \sum_{i=i_1}^{i_2-1} \sum_{\ell \in J_i^k} \Vert \tilde{x}_{i+1,\ell}^k - \tilde{x}_{i,\ell}^k \Vert^2  \leq f(\mathbf{\tilde{x}}_{i_1}^k) + m_k(\mathbf{\tilde{x}}^{k}_{i_1}).
\end{equation*}
By adding as many times as necessary $\sum_{\ell=1}^L (\phi_\ell \circ \psi_\ell)(\tilde{x}_\ell^k) - (\phi_\ell' \circ \psi_\ell)(\psi_\ell(\tilde{x}_\ell^k))$, which is constant over $i$, on both sides, we recover Equation \eqref{eq:descent_cycle_majorant}.
\end{proof}
\noindent We are now ready to look at the descent properties of our algorithm.
\begin{proposition}
\label{prop:descent_cycle} Let $(\mathbf{x}^k)_{k\in\mathbb{N}}$ and, for every $k \in \mathbb{N}$, $(\mathbf{\tilde{x}}_i^k)_{0\leq i \leq I_k}$ be the iterates generated by Algorithm \ref{alg:MAJ-BC-FB}. Under Assumptions \ref{ass:1}, \ref{ass:3}, \ref{ass:2}, \ref{ass:4}, and \ref{ass:bounded} the following holds
\begin{enumerate}[(ii)]
    \item For every $k \in \mathbb{N}$, we have
    \begin{equation} \label{eq:sufficient_decrease1}
        \Psi(\mathbf{x}^{k+1}) \leq \Psi_k(\mathbf{x}^{k+1}) \leq \Psi(\mathbf{x}^{k}) - \eta \sum_{i=0}^{I_k-1} \sum_{\ell \in J_i^k} \Vert \tilde{x}_{i+1,\ell}^k - \tilde{x}_{i,\ell}^k \Vert^2,
    \end{equation}
    where $\eta>0$ is given in Proposition \ref{prop:descent_cycle_majorant}. We can deduce that
    \begin{equation}
        \Psi(\mathbf{x}^{k+1}) \leq \Psi(\mathbf{x}^k) - \eta \bar{I}^{-1} \vertiii{\mathbf{x}^{k+1} - \mathbf{x}^k}^2.
    \end{equation}
    \item $\Psi(\mathbf{x}^k)$ is a converging non-increasing sequence.
    \item We have
    \begin{align}
        \sum_{k\in\mathbb{N}} \left( \sum_{i=0}^{I_k-1} \sum_{\ell \in J_i^k} \Vert \tilde{x}_{i+1,\ell}^k - \tilde{x}_{i,\ell}^k \Vert^2 \right)< +\infty, \text{ and }
        \sum_{k\in\mathbb{N}}\vertiii{\mathbf{x}^{k+1} - \mathbf{x}^k}^2 < +\infty,
    \end{align}
    therefore
    \begin{align}
        \lim_{k\to+\infty} \left(\sum_{i=0}^{I_k-1} \sum_{\ell \in J_i^k} \Vert \tilde{x}_{i+1,\ell}^k - \tilde{x}_{i,\ell}^k \Vert \right) = 0, \text{ and }
        \lim_{k\to+\infty} \vertiii{\mathbf{x}^{k+1} - \mathbf{x}^k} = 0.
    \end{align}
\end{enumerate}
\end{proposition}
\begin{proof}
    \begin{enumerate}[(ii)]
        \item The first inequality is a consequence of Proposition \ref{prop:descent_cycle_majorant}. Set $i_1 = 0$ and $i_2= I_k$ in Equation \eqref{eq:descent_cycle_majorant}. Notice that $\mathbf{\tilde{x}}_{I_k}^k = \mathbf{x}^{k+1}$, and $\mathbf{\tilde{x}}_0^k = \mathbf{x}^k$. We have $\Psi_k(\mathbf{x}^{k}) = \Psi(\mathbf{x}^{k})$, thus
        \begin{align*}
            \Psi_k(\mathbf{x}^{k+1}):=\Psi_k(\mathbf{\tilde{x}}_{I_k}^k)   & \leq \Psi_k(\mathbf{\tilde{x}}_{0}^k) - \eta \sum_{i=0}^{I_k-1} \sum_{\ell \in J_i^k} \Vert \tilde{x}_{i+1,\ell}^k - \tilde{x}_{i,\ell}^k \Vert^2, \\
            & = \Psi(\mathbf{x}^k)- \eta \sum_{i=0}^{I_k-1} \sum_{\ell \in J_i^k} \Vert \tilde{x}_{i+1,\ell}^k - \tilde{x}_{i,\ell}^k \Vert^2,
        \end{align*}
        and using that $\Psi(\mathbf{x}^{k+1}) \leq \Psi_k(\mathbf{x}^{k+1})$, we obtain
        \begin{equation*}
            \Psi (\mathbf{x}^{k+1})  + \eta \sum_{i=0}^{I_k-1} \sum_{\ell \in J_i^k} \Vert \tilde{x}_{i+1,\ell}^k - \tilde{x}_{i,\ell}^k \Vert^2 \leq \Psi(\mathbf{x}^k).
        \end{equation*}
        Now, we will use the Jensen's inequality to obtain the second inequality. We have
        \begin{align*}
            \vertiii{\mathbf{x}^{k+1} - \mathbf{x}^k}^2  = \vertiii{\sum_{i=0}^{I_k-1}\mathbf{\tilde{x}}_{i+1}^k - \mathbf{\tilde{x}}_{i}^k }^2 
            & \leq I_k \sum_{i=0}^{I_k-1} \vertiii{\mathbf{\tilde{x}}_{i+1}^k - \mathbf{\tilde{x}}_{i}^k }^2 
            = I_k \sum_{i=0}^{I_k-1} \sum_{\ell \in J_i^k} \Vert \tilde{x}_{i+1,\ell}^k - \tilde{x}_{i,\ell}^k \Vert^2.
        \end{align*}
        Thus,
        \begin{equation}
            \frac{1}{\bar{I}}\vertiii{\mathbf{x}^{k+1} - \mathbf{x}^k}^2 \leq \sum_{i=0}^{I_k-1} \sum_{\ell \in J_i^k} \Vert \tilde{x}_{i+1,\ell}^k - \tilde{x}_{i,\ell}^k \Vert^2,
        \end{equation}
        which yields the desired result.
        \item From the first point, $(\Psi(\mathbf{x}^k))_{k\in\mathbb{N}}$ is a non-increasing sequence. Furthermore, by Assumption \ref{ass:bounded}, the iterates are bounded. 
        As $\Psi$ is continuous on its domain, it is also bounded from below by a real number. Therefore, $(\Psi(\mathbf{x}^k))_{k\in\mathbb{N}}$ converges to a real number.
        \item Let $K$ be a positive integer. Summing up inequality \eqref{eq:sufficient_decrease1} for k=0 to $K-1$, we obtain
        \begin{align*}
            \eta \sum_{k=0}^{K-1} \sum_{i=0}^{I_k-1} \sum_{\ell \in J_i^k} \Vert \tilde{x}_{i+1,\ell}^k - \tilde{x}_{i,\ell}^k \Vert^2 & \leq \Psi(\mathbf{x}^K) - \Psi(\mathbf{x}^0) \\
            & \leq  \Psi(\mathbf{x}^K) - \inf \Psi
        \end{align*}
        Taking the limit when $K$ goes to infinity yields
        \begin{equation}
            \frac{1}{\bar{I}} \sum_{k\in\mathbb{N}}\vertiii{\mathbf{x}^{k+1} - \mathbf{x}^k}^2 \leq\sum_{k\in\mathbb{N}} \sum_{i=0}^{I_k-1} \sum_{\ell \in J_i^k} \Vert \tilde{x}_{i+1,\ell}^k - \tilde{x}_{i,\ell}^k \Vert^2 < +\infty,
        \end{equation}
        which is the desired result.
    \end{enumerate}
\end{proof}
The next proposition upper bounds the norm of a sequence of subgradients.
\begin{proposition}{\textbf{Subgradient norm upper bound.}} \label{prop:subgradient_bound}
Let $(\mathbf{x}^k)_{k\in\mathbb{N}}$ and, for every $k \in \mathbb{N}$, $(\mathbf{\tilde{x}}_i^k)_{0\leq i \leq I_k}$ be the iterates generated by Algorithm \ref{alg:MAJ-BC-FB}. Let $k \in \mathbb{N}$ and $0 \leq i_1, < i_2 \leq I_k$. Under Assumptions \ref{ass:1}, \ref{ass:3}, \ref{ass:2}, and \ref{ass:4}, we have for every $k \in \mathbb{N}$,
\begin{equation}
    \vertiii{t(\mathbf{x}^{k+1})} \leq \kappa \sum_{i=0}^{I_k-1} \sum_{\ell \in J_i^k} \Vert \tilde{x}_{i+1,\ell}^k - \tilde{x}_{i,\ell}^k \Vert,
\end{equation}
where $\kappa>0$, and
\begin{equation}
    \left\{ \begin{array}{l}
        t(\mathbf{x}^{k+1}) = \nabla f (\mathbf{x}^{k+1}) + v^{k+1} \in \partial \Psi(\mathbf{x}^{k+1}) \\
        v^{k+1}(\mathbf{x}^{k+1}) = \left( \lambda_{\ell,k+1} v_\ell^{k+1}(x_\ell^{k+1})\right)_{1\leq \ell \leq L}
    \end{array}\right.
\end{equation}
For every $\ell \in \{1,\ldots,L\}$, $v_\ell^{k+1}(x_\ell^{k+1}) = \tilde{v}_\ell^{k}(\tilde{x}_{k_\ell+1,\ell}^k) \in \partial  \psi_\ell(\tilde{x}_{k_\ell+1,\ell}^k)$, given by Algorithm \ref{alg:MAJ-BC-FB}, at the last iteration $k_\ell \in \{0,\ldots,I_k-1\}$ where the block $\ell$ has been updated.
\end{proposition}
\begin{proof}
    We use similar arguments to the proof \citep[Proposition 4.3]{repetti2021variable}. We have that $\partial \Psi(\mathbf{x}^{k+1}) = \nabla f (\mathbf{x}^{k+1}) + \partial g(\mathbf{x}^{k+1})$. Now, using Propositions \ref{prop:separability} and \ref{prop:chain_rule}, we have that \begin{align}
\partial g(\mathbf{x}^{k+1}) = \left( \lambda_{\ell,k+1} \partial \psi_\ell(x_{\ell}^{k+1})\right)_{1\leq \ell \leq L} = \left( \lambda_{\ell,k+1} \partial \psi_\ell(\tilde{x}_{k_\ell+1,\ell}^{k})\right)_{1\leq \ell \leq L},
\end{align}
where $k_\ell$ is the last iteration where the block $\ell$ has been updated, and where $\lambda_{\ell,k+1} = (\phi'_\ell \circ \psi_\ell)(x_{\ell}^{k+1}) = (\phi'_\ell \circ \psi_\ell)(\tilde{x}_{k_\ell+1,\ell}^{k})$. Moreover, for all $\ell \in \{1,\ldots,L\}$, we have
\begin{align}
    \Vert \nabla_\ell f(\mathbf{x}^{k+1}) + v_\ell^{k+1}(x_\ell^{k+1}) \Vert & =  \Vert \nabla_\ell f(\mathbf{x}^{k+1}) - \nabla_\ell f(\mathbf{\tilde{x}}_{k_\ell}^k) + \nabla_\ell f(\mathbf{\tilde{x}}_{k_\ell}^k) + v_\ell^{k+1}(x_\ell^{k+1}) \Vert \nonumber \\
    & \leq \Vert \nabla_\ell f(\mathbf{x}^{k+1}) - \nabla_\ell f(\mathbf{\tilde{x}}_{k_\ell}^k) \Vert + \Vert \nabla_\ell f(\mathbf{\tilde{x}}_{k_\ell}^k) + v_\ell^{k+1}(x_\ell^{k+1}) \Vert  \nonumber \\
    & \leq \Vert \nabla_\ell f(\mathbf{x}^{k+1}) - \nabla_\ell f(\mathbf{\tilde{x}}_{k_\ell}^k) \Vert + \mu \Vert \tilde{x}_{k_\ell+1,\ell}^k - \tilde{x}_{k_\ell,\ell}^k \Vert_{A^k_{i,\ell}}
\end{align}
Hence, we can use the block smoothness of $f$ to obtain
\begin{align}
    \Vert \nabla_\ell f(\mathbf{x}^{k+1})-\nabla_\ell f(\mathbf{\tilde{x}}_{k_\ell}^k) \Vert & \leq \sum_{j=1}^L \beta_{\ell,j} \Vert x^{k+1}_j - \tilde{x}^{k}_{k_\ell,j} \Vert \nonumber \\
    & = \sum_{j=1}^L \beta_{\ell,j} \left\Vert \sum_{i=k_\ell}^{I_k-1} \tilde{x}_{i+1,j}^{k} - \tilde{x}_{i,j}^{k} \right\Vert \nonumber \\
    & \leq \sum_{j=1}^L \beta_{\ell,j} \sum_{i=k_\ell}^{I_k-1} \Vert \tilde{x}_{i+1,j}^{k} - \tilde{x}_{i,j}^{k} \Vert \nonumber \\
    & \leq \max_{\ell,j} \beta_{\ell,j} \sum_{j=1}^L  \sum_{i=0}^{I_k-1} \Vert \tilde{x}_{i+1,j}^{k} - \tilde{x}_{i,j}^{k} \Vert \nonumber \\
    & = \max_{\ell,j} \beta_{\ell,j} \sum_{i=0}^{I_k-1}  \sum_{j\in J_i^k}\Vert \tilde{x}_{i+1,j}^{k} - \tilde{x}_{i,j}^{k} \Vert.
\end{align}
Now set $\kappa = L \times \left(\max_{\ell,j} \beta_{\ell,j} + \mu \overline{\nu}\right)$ to obtain the desired bound on the norm of the subgradient:
\begin{equation}
    \vertiii{t(\mathbf{x}^{k+1})} \leq \kappa \sum_{i=0}^{I_k-1}  \sum_{\ell \in J_i^k} \Vert \tilde{x}_{i+1,\ell}^k - \tilde{x}_{i,\ell}^k \Vert.
\end{equation}
\end{proof}

\paragraph{Convergence to a critical point.} \label{subsec:critical_point_gen}
We are now ready to show that the iterates generated by Algorithm \ref{alg:MAJ-BC-FB} converge to a critical point of the objective function. 
Assumption \ref{ass:5} combined with descent properties of the algorithm, are nearly enough to show convergence of the algorithm. 

The proof of the convergence of the sequence requires the study of the limit points set, defined as follows.  
\begin{definition}{\textbf{Limit points set \citep{bolte2014proximal}.}} \label{def:limit_point_set}
 The set of all limit points of sequences generated by Algorithm \ref{alg:MAJ-BC-FB} from a starting point $\mathbf{x}^0$ will be denoted by $\lp(\mathbf{x}^0)$:
    \begin{align*}
        \lp(\mathbf{x}^0) = \{\widehat{\mathbf{x}} \in \Hi, \exists & \text{ an increasing sequence of integers } \{k_j\}_{j\in\mathbb{N}},  \\ &\text{ such that } \mathbf{x}^{k_j} \rightarrow \widehat{\mathbf{x}} \text{ as } j \rightarrow +\infty \}
    \end{align*}
\end{definition}
The properties of the limit points of sequences produced by block algorithms such as the proposed MAJ-BC-FB were investigated for instance in \citep{bolte2014proximal,lauga2025BCD,repetti2021variable}. The proofs are similar, and we simply reproduce the most important components.  
\begin{lemma}{\textbf{Properties of the limit points set.}}
    \label{lm:limit_point_set}
    Suppose that Assumptions \ref{ass:1}, \ref{ass:3}, \ref{ass:2}, \ref{ass:4} and \ref{ass:5} hold.
    Let $\{\mathbf{x}^k\}_{k \in \mathbb{N}}$ be a sequence generated by Algorithm \ref{alg:MAJ-BC-FB} starting from $\mathbf{x}^0$. %
    The following hold:
    \begin{enumerate}[(ii)]
        \item $\emptyset \neq \lp(\mathbf{x}^0) \subset$ crit $\Psi$.
        \item We have
        \begin{equation}
            \lim_{k \rightarrow \infty} \mathrm{dist} (\mathbf{x}^k,\lp(\mathbf{x}^0)) = 0.
        \end{equation}
        \item $\lp(\mathbf{x}^0)$ is a nonempty, compact and connected set.
        \item The objective function $\Psi$ is finite and constant on $\lp(\mathbf{\bar{x}}^0)$.
    \end{enumerate}
\end{lemma}
\begin{proof}
    \begin{enumerate}[(ii)]
        \item Let $\bar{\mathbf{x}}$ be a limit point of $\{\mathbf{x}^k\}_{k \in \mathbb{N}}$. Since for every $k\in\mathbb{N}, \mathbf{x}^k \in \dom g$, and since $\Psi$ is continuous on $\dom g$, by Definition \ref{def:limit_point_set} there exists a subsequence  $\{\mathbf{x}^{k_q}\}_{q \in \mathbb{N}}$ such that $\mathbf{x}^{k_q} \rightarrow \bar{\mathbf{x}}$ and such that $\Psi(\mathbf{x}^{k_q}) \rightarrow \Psi(\bar{\mathbf{x}})$ as $q$ goes to infinity. Moreover, the $\{\Psi(\mathbf{x}^k)\}_{k \in \mathbb{N}}$ is non-increasing and converges (Proposition \ref{prop:descent_cycle}, point (ii)), hence $\Psi(\mathbf{x}^{k}) \rightarrow \Psi(\bar{\mathbf{x}})$ as $k\rightarrow + \infty$.

        \noindent Now, we know from Proposition \ref{prop:subgradient_bound} and Proposition \ref{prop:descent_cycle} point (iii) that $t(\mathbf{x}^k) \rightarrow 0$ as $k\rightarrow + \infty$. The closedness property of $\partial \Psi$ implies that $0\in\partial \Psi( \bar{\mathbf{x}})$, and therefore $\bar{\mathbf{x}}$ is a critical point of $\Psi$.
        \item We have from point (i) of Proposition \ref{prop:descent_cycle} that $\sum_{i=0}^{I_k-1} \sum_{\ell \in J_i^k} \Vert \tilde{x}_{i+1,\ell}^k - \tilde{x}_{i,\ell}^k \Vert \rightarrow 0$ as $k$ goes to infinity. Therefore $\lim_{k \rightarrow \infty} \mathrm{dist} (\mathbf{x}^k,\lp(\mathbf{x}^0)) = 0$.
        \item  Identical to \citep[Remark 5 \& Lemma 5]{bolte2014proximal}, given point (i) of Proposition \ref{prop:descent_cycle}.
        \item It is a direct consequence from point 1 (see \citep[Lemma 5]{bolte2014proximal}).
    \end{enumerate}
\end{proof}
\noindent We are now ready to state our main convergence result:
\begin{theorem} \label{th:convergence} Suppose that Assumptions \ref{ass:1}, \ref{ass:3}, \ref{ass:2},  \ref{ass:4}, \ref{ass:bounded}, and \ref{ass:5} hold.
    Let $\{\mathbf{x}^k\}_{k \in \mathbb{N}}$ be a sequence generated by Algorithm \ref{alg:MAJ-BC-FB}. The following hold
    \begin{enumerate}[(ii)]
        \item $\{\mathbf{x}^k\}_{k \in \mathbb{N}}$ has finite length, i.e., $\sum_{k=0}^{+\infty} \vertiii{\mathbf{x}^{k+1}-\mathbf{x}^k} < + \infty$.
        \item $\{\mathbf{x}^k\}_{k \in \mathbb{N}}$ converges to a critical point $\widehat{\mathbf{x}}$ of $\Psi$.
    \end{enumerate}
\end{theorem}
\begin{proof} The proof is identical to the one of \citep[Theorem 3]{lauga2025BCD}, up to the naming of the constants. It is stated in Appendix \ref{app:proofs} to facilitate the reading of the paper as it is mostly composed of computations.
\end{proof}
This analysis can be carried out to other algorithm as long as sufficient decrease condition and a subgradient bound coexist at each iteration. To solve the sparse precision matrix estimation problem, we carry this convergence analysis on two other algorithms: a BC proximal Newton algorithm and a Gauss-Seidel algorithm.
\subsection{A proximal Newton algorithm for composite optimization} In this section, we present how to incorporate second order information in our framework. Assuming that $f$ is twice continuously differentiable with positive definite Hessian, we can incorporate Newton-like update in our algorithm. A Newton variant of our algorithm is obtained by replacing the variable metric $A_{i}^k$ by the Hessian or an approximation of the Hessian matrix of $f$ (as long as this approximation is positive definite). To guarantee the positive definiteness, we will assume strong convexity of $f$.
\begin{assumption} \label{ass:8}
    The function $f$ is twice continuously differentiable, and strongly convex, i.e., there exists $\underline{\nu},\overline{\nu}>0$ such that for all $\mathbf{x}\in\Hi$, $ \underline{\nu} \preccurlyeq \nabla^2 f (\mathbf x) \preccurlyeq  \overline{\nu} \Id$.
\end{assumption}
The block-variable metric can be defined from the Hessian matrix by taking for all $k\in \mathbb{N}$, $i\in\{0,\ldots,I_k-1\}$, $\ell \in \{1,\ldots,L\}$, $
    A^k_{i,\ell} := \nabla^2_{\ell,\ell} f(\mathbf{\tilde{x}}^k_i)$. 
    
\noindent However, a typical proximal-Newton iteration \cite{lee2014proximal} on the majorant $\Psi_k$ writes
\begin{equation*}
    \Delta \mathbf{x}^k = \argmin_{\mathbf{d}\in\Hi} \langle \langle \nabla f(\mathbf{x}^k),\mathbf{d}\rangle \rangle + \frac{1}{2} \langle \langle \mathbf{d},\nabla^2 f(\mathbf{x}^k) \mathbf{d} \rangle \rangle + \sum_{\ell=1}^L \lambda_{\ell,k} \psi_\ell({x}_{\ell}^k+d_\ell),
\end{equation*}
followed by a line search to find $\alpha^k>0$ such that $\mathbf{x}^{k+1} = \mathbf{x}^k + \alpha^k \Delta \mathbf{x}^k$ decreases $\Psi_k$ \citep{lee2014proximal}. With a minimal modification, we can fit these proximal Newton updates in our framework. First, we replace the minimization on $\Psi_k$ by an damped formulation
\begin{equation} \label{eq:prox_Newt_full}
    \mathbf{\tilde{x}}^k_{i+1} = \argmin_{\mathbf{y}\in \Hi} \langle \langle \nabla f(\mathbf{\tilde{x}}^k_{i}),\mathbf{y}-\mathbf{\tilde{x}}^k_{i}\rangle \rangle + \frac{1}{2\alpha^k} \langle \langle \mathbf{y}-\mathbf{\tilde{x}}^k_{i},\nabla^2 f(\mathbf{\tilde{x}}^k_{i}) (\mathbf{y}-\mathbf{\tilde{x}}^k_{i} )\rangle \rangle +m_k(\mathbf{y}) 
\end{equation}
Then, as the proximal-Newton problem may be solved with block-coordinate methods \citep{lee2014proximal,hsieh2013big}, we write this minimization directly as coordinate descent method. This is again straightforward with our framework as we allow updating several blocks simultaneously. 
Let $J_i^k \subset \{1,\ldots,L\}$. We jointly solve for all $\ell \in J_i^k$ the proximal Newton problem, i.e., we find
\begin{equation}
    \tilde{x}_{i+1,J_i^k}^{k} = \argmin_{y\in\mathbin{\scalebox{1}{$\times$}}_{\ell\in J_i^k}\Hi_\ell}  \langle \nabla_{J_i^k} f(\mathbf{\tilde{x}}_i^k),y-\tilde{x}_{J_i^k}^k \rangle + \frac{1}{2 \alpha_i^k} \langle y-\tilde{x}_{i,J_i^k}^{k},\nabla^2_{J_i^k,J_i^k} f(\mathbf{\tilde{x}}_i^k) (y-\tilde{x}_{i,J_i^k}^{k})  \rangle + \sum_{\ell \in J_i^k} \lambda_{\ell,k}\psi_\ell(y_\ell),
\end{equation}
Note that solving at the same time for $J_i^k$ and $\overline{J_i^k}$ with this formulation is not equivalent as solving \eqref{eq:prox_Newt_full}. Indeed, here we implicitly assume that other blocks are kept fixed. The resulting algorithm is described in Algorithm \ref{alg:NEW-BC-FB} as CO-NEWton-BC-FB.
\begin{algorithm}[t]
\caption{\textbf{CO-NEW-BC-FB}}
\label{alg:NEW-BC-FB}
\begin{algorithmic}[1]
\For{$k = 0,1,\dots$}
  \State $\mathbf{\tilde{x}}_0^k \gets \mathbf{x}^k$
  \For{$i = 0,\ldots, I_k-1$}
    \State \textbf{Find} $\alpha_i^k>0$, $\mathbf{\tilde{x}}_{i+1}^k$, and
    $\tilde{v}_i^k(\mathbf{\tilde{x}}_{i+1}^k)\in \partial_{J_i^k} m_k(\mathbf{\tilde{x}}^k_{i+1,J_i^k})$ such that
    \Statex \hspace{\algorithmicindent}%
    $\tilde{x}^k_{i+1,J_i^k}
      = \argmin_{y\in \mathbin{\scalebox{1}{$\times$}}_{\ell\in J_i^k}\Hi_\ell}
      \Big\langle \nabla_{J_i^k} f(\mathbf{\tilde{x}}^k_i),\, y-\mathbf{\tilde{x}}^k_{i,J_i^k}\Big\rangle
      + \frac{1}{2\alpha_i^k}\Big\langle y-\mathbf{\tilde{x}}^k_{i,J_i^k},\,
      \nabla^2_{J_i^k,J_i^k} f(\mathbf{\tilde{x}}^k_{i})
      (y-\mathbf{\tilde{x}}^k_{i,J_i^k})\Big\rangle
      + \sum_{\ell\in J_i^k}\lambda_{\ell,k}\psi_\ell(y_\ell)$
    \Statex \hspace{\algorithmicindent}%
    $\tilde{x}^k_{i+1,\overline{J_i^k}} \gets \tilde{x}^k_{i,\overline{J_i^k}}$
    \Statex \hspace{\algorithmicindent}%
    $\Psi_k(\mathbf{\tilde{x}}_{i+1}^k) + \eta\,\vertiii{\mathbf{\tilde{x}}_{i+1}^k-\mathbf{\tilde{x}}_i^k}^{\,2}
      \le \Psi_k(\mathbf{\tilde{x}}_i^k)$
    \Statex \hspace{\algorithmicindent}%
    $\big\|\nabla_{J_i^k} f(\mathbf{\tilde{x}}_i^k) + \tilde{v}_i^k(\mathbf{\tilde{x}}_{i+1}^k)\big\|
      \le \mu\,\sum_{\ell\in J^k_i}\big\|\tilde{x}^k_{i+1,\ell}-\tilde{x}^k_{i,\ell}\big\|$
  \EndFor
  \State $\mathbf{x}^{k+1} \gets (\tilde{x}_{I_k,1}^{k},\ldots,\tilde{x}_{I_k,L}^{k})$
\EndFor
\end{algorithmic}
\end{algorithm}

To derive the convergence analysis of this algorithm, we start by showing the existence of a non-zero and bounded step sizes at each $k\in \mathbb{N}$, $i\in \{0,\ldots,I_{k}-1\}$ such that there exists $\eta>0$ and $\Psi_k(\mathbf{\tilde{x}}_{i+1}^k) + \eta\vertiii{\mathbf{\tilde{x}}_{i+1}^k-\mathbf{\tilde{x}}_i^k}^2\leq \Psi_k(\mathbf{\tilde{x}}_i^k)$.
\begin{proposition}{\textbf{Sufficient decrease.}} \label{prop:newton_descent_cycle_majorant}
    Let $(\mathbf{x}^k)_{k\in\mathbb{N}}$ and, for every $k \in \mathbb{N}$, $(\mathbf{\tilde{x}}_i^k)_{0\leq i \leq I_k}$ be the iterates generated by Algorithm \ref{alg:NEW-BC-FB}. Let $k \in \mathbb{N}$ and $0 \leq i_1, < i_2 \leq I_k$. Let $\gamma \in (0,1/2)$. Under Assumptions \ref{ass:1}, \ref{ass:2}, \ref{ass:4}, and \ref{ass:8}, there exist 
    \begin{equation}
    0< \alpha_i^k \leq \min\left\{ 1, (1-\gamma)\frac{\underline{\nu}}{\beta}\right\},
\end{equation}
then there exists $\eta>0$ such that
\begin{equation}
\Psi_k(\mathbf{\tilde{x}}_{i+1}^k) + \eta\sum_{\ell \in J_i^k} \Vert \tilde{x}_{i+1,\ell}^k - \tilde{x}_{i,\ell}^k \Vert^2 \leq \Psi_k(\mathbf{\tilde{x}}_i^k),
\end{equation}
hence
\begin{equation} \label{eq:new_descent_cycle_majorant}
    \Psi_k(\mathbf{\tilde{x}}_{i_2}^k) + \eta \sum_{i=i_1}^{i_2-1} \sum_{\ell \in J_i^k} \Vert \tilde{x}_{i+1,\ell}^k - \tilde{x}_{i,\ell}^k \Vert^2  \leq \Psi_k(\mathbf{\tilde{x}}_{i_1}^k).
\end{equation}
\end{proposition}
\begin{proof}
    We have for every $k\in \mathbb{N}$, $i\in \{0,\ldots,I_k\}$, that $\tilde{x}_{i+1}^k$, as a minimizer, ensures %
\begin{equation}\label{eq:lambda_bound}
    \langle  \nabla_{J_i^k} f(\mathbf{\tilde{x}}^k_i), \tilde{x}^k_{i+1,J_i^k} - \tilde{x}^k_{i,J_i^k} \rangle + m_k(\tilde{x}^k_{i+1,J_i^k}) -m_k(\tilde{x}^k_{i,J_i^k}) \leq \frac{-1}{2 \alpha^k_i}  \langle \tilde{x}^k_{i+1,J_i^k} - \tilde{x}^k_{i,J_i^k},\nabla^2_{J_i^k,J_i^k} f(\mathbf{\tilde{x}}^k_i)(\tilde{x}^k_{i+1,J_i^k} - \tilde{x}^k_{i,J_i^k}) \rangle,
\end{equation}
where for the simplicity of reading we replace $\sum_{\ell \in J_i^k} \lambda_{\ell,k}\psi_\ell(y_\ell)$ by $m_k(y_{J_i^k})$, and whenever it is needed, $\sum_{\ell=1}^L (\phi_\ell \circ \psi_\ell)(\tilde{x}_\ell^k) - (\phi_\ell' \circ \psi_\ell)(\psi_\ell(\tilde{x}_\ell^k))$ is added implicitly as many times as necessary on both sides of the inequalities (as it is constant over $i$).
Set $\gamma\in(0,1/2)$. An Armijo line-search \citep{bertsekas1999nonlinear,hsieh2014quic,lee2014proximal} consists in finding $\alpha^k_i$ such that 
\begin{equation*}
    \Psi_k(\mathbf{\tilde{x}}^k_{i+1}) - \Psi_k(\mathbf{\tilde{x}}^k_i)\leq \gamma \left( \langle  \nabla_{J_i^k} f(\mathbf{\tilde{x}}^k_i), \tilde{x}^k_{i+1,J_i^k} - \tilde{x}^k_{i,J_i^k}  \rangle  + m_k(\mathbf{\tilde{x}}^k_{i+1}) -m_k(\mathbf{\tilde{x}}^k_{i})\right),
\end{equation*}
We have by Assumption \ref{ass:1} that there exists $\beta_{J_i^k}>0$ such that 
\begin{equation*}
    \Psi_k(\mathbf{\tilde{x}}^k_{i+1}) - \Psi_k(\mathbf{\tilde{x}}^k_i)\leq \langle  \nabla_{J_i^k} f(\mathbf{\tilde{x}}^k_i), \tilde{x}^k_{i+1,J_i^k} - \tilde{x}^k_{i,J_i^k}  \rangle + m_k(\mathbf{\tilde{x}}^k_{i+1}) -m_k(\mathbf{\tilde{x}}^k_{i}) + \frac{\beta_{J_i^k}}{2}\Vert \tilde{x}^k_{i+1,J_i^k} - \tilde{x}^k_{i,J_i^k}\Vert^2.
\end{equation*}
Thus, we want
\begin{equation*}
    (1-\gamma)\left(\langle  \nabla_{J_i^k} f(\mathbf{\tilde{x}}^k_i), \tilde{x}^k_{i+1,J_i^k} - \tilde{x}^k_{i,J_i^k}  \rangle + m_k(\mathbf{\tilde{x}}^k_{i+1}) -m_k(\mathbf{\tilde{x}}^k_{i})\right) + \frac{\beta_{J_i^k}}{2}\Vert \tilde{x}^k_{i+1,J_i^k} - \tilde{x}^k_{i,J_i^k}\Vert^2 \leq 0.
\end{equation*}
An by \eqref{eq:lambda_bound} we find that this holds if
\begin{equation*}
    \left(\frac{\beta_{J_i^k}}{2}-\frac{(1-\gamma)}{2\alpha^k_i}\underline{\nu}\right) \Vert\tilde{x}^k_{i+1,J_i^k} - \tilde{x}^k_{i,J_i^k}\Vert^2 \leq 0,
\end{equation*}
which holds if
\begin{equation*}
    0< \alpha^k_i \leq \min\left\{ 1, (1-\gamma)\frac{\underline{\nu}}{\beta_{J_i^k}}\right\}.
\end{equation*}
Therefore, we have the existence of $\eta>0$ such that
\begin{equation*}
    \Psi_k(\mathbf{\tilde{x}}_{i+1}^k) + \eta\sum_{\ell \in J_i^k} \Vert \tilde{x}_{i+1,\ell}^k - \tilde{x}_{i,\ell}^k \Vert^2 \leq \Psi_k(\mathbf{\tilde{x}}_i^k).
\end{equation*}
\end{proof}

We can expect that $\alpha^k=1$ is sufficient when a solution is close \citep{hsieh2014quic} (i.e., when the difference between iterates is small). Now, an identical proposition to Proposition \ref{prop:descent_cycle} holds (i.e., a sufficient decrease condition on $\Psi$) and for conciseness of the argument we won't restate it here.

\begin{proposition}{\textbf{Subgradient norm upper bound.}}\label{prop:newton_subgradient_bound}
Let $(\mathbf{x}^k)_{k\in\mathbb{N}}$ and, for every $k \in \mathbb{N}$, $(\mathbf{\tilde{x}}_i^k)_{0\leq i \leq I_k}$ be the iterates generated by Algorithm \ref{alg:NEW-BC-FB}. Let $k \in \mathbb{N}$ and $0 \leq i_1, < i_2 \leq I_k$. Under Assumptions \ref{ass:1}, \ref{ass:2}, \ref{ass:4}, \ref{ass:5} and \ref{ass:8}, we have for every $k \in \mathbb{N}$,
\begin{equation}
    \vertiii{t(\mathbf{x}^{k+1})} \leq \kappa \sum_{i=0}^{I_k-1} \sum_{\ell \in J_i^k} \Vert \tilde{x}_{i+1,\ell}^k - \tilde{x}_{i,\ell}^k \Vert,
\end{equation}
where $\kappa>0$, and
\begin{equation}
    \left\{ \begin{array}{l}
        t(\mathbf{x}^{k+1}) = \nabla f (\mathbf{x}^{k+1}) + v^{k+1} \in \partial \Psi(\mathbf{x}^{k+1}) \\
        v^{k+1}(\mathbf{x}^{k+1}) = \left( \lambda_{\ell,k+1} v_\ell^{k+1}(x_\ell^{k+1})\right)_{1\leq \ell \leq L}
    \end{array}\right.
\end{equation}
For every $\ell \in \{1,\ldots,L\}$, $v_\ell^{k+1}(x_\ell^{k+1}) = \tilde{v}_\ell^{k}(\tilde{x}_{k_\ell+1,\ell}^k) \in \partial  \psi_\ell(\tilde{x}_{k_\ell+1,\ell}^k)$, given by Algorithm \ref{alg:NEW-BC-FB}, at the last iteration $k_\ell \in \{0,\ldots,I_k-1\}$ where the block $\ell$ has been updated.
\end{proposition}
\begin{proof}The proof is essentially the same as the one of Proposition \ref{prop:subgradient_bound}. We only need to show that the selected subgradient in Algorithm \ref{alg:NEW-BC-FB} can indeed satisfy the bound.
    We have for every $k\in \mathbb{N}$, $i\in \{0,\ldots,I_k\}$, that $\tilde{x}_{i+1}^k$, as a minimizer, ensures  yields the following subdifferential inclusion
\begin{equation*}
    \frac{1}{\alpha^k_i} \nabla^2_{J_i^k,J_i^k} f(\mathbf{\tilde{x}}^k_i)(\tilde{x}^k_{i,J_i^k}-\tilde{x}^k_{i+1,J_i^k}) \in \nabla_{J_i^k} f (\mathbf{\tilde{x}}^k_i) + \partial_{J_i^k} m_k(\tilde{x}^k_{i+1,J_i^k}). 
\end{equation*}
Hence,
\begin{align*}
\frac{1}{\alpha^k_i} \Vert \nabla^2_{J_i^k,J_i^k} f(\mathbf{\tilde{x}}^k_i)(\tilde{x}^k_{i,J_i^k}-\tilde{x}^k_{i+1,J_i^k})\Vert &\leq \frac{\overline{\nu}}{\alpha^k_i} \Vert \tilde{x}^k_{i,J_i^k}-\tilde{x}^k_{i+1,J_i^k}\Vert = \frac{\overline{\nu}}{\alpha^k_i} \sqrt{\sum_{\ell \in J^k_i} \Vert \tilde{x}^k_{i+1,\ell}-\tilde{x}^k_{i,\ell}\Vert^2} \\ & \leq \frac{\overline{\nu}}{\alpha^k_i} \sum_{\ell \in J^k_i} \Vert \tilde{x}^k_{i+1,\ell}-\tilde{x}^k_{i,\ell}\Vert 
\end{align*}
Now there exists some $\mu>0$ such that $\alpha_i^k\geq \mu$ for all $k\in\mathbb{N}$, $i\in \{0,\ldots,I_k-1\}$. Take for all $k\in\mathbb{N}$, $i\in \{0,\ldots,I_k-1\}$
\begin{equation*}
    (\tilde{v}_\ell^{k}(\mathbf{\tilde{x}}^k_{i+1}))_{\ell \in J^k_i} = \left(\frac{1}{\alpha^k_i}\nabla^2_{J_i^k,J_i^k} f(\mathbf{\tilde{x}}^k_i)(\tilde{x}^k_{i,J_i^k}-\tilde{x}^k_{i+1,J_i^k})-\nabla_{J_i^k} f (\mathbf{\tilde{x}}^k_i)\right)_{\ell \in J^k_i}.
\end{equation*}
From this point the proof is the same as in Proposition \ref{prop:subgradient_bound} and set $\kappa = L \times (\overline{\nu}+\mu)$.
\end{proof}
\noindent With these two propositions (\ref{prop:newton_descent_cycle_majorant} and \ref{prop:newton_subgradient_bound}), Theorem \ref{th:convergence} holds for Algorithm \ref{alg:NEW-BC-FB}. We therefore have convergence of the iterates to a critical point of $\Psi$.

\subsection{A Gauss-Seidel algorithm for composite optimization}
\label{subsec:Gauss-Seidel}
In this section, we present the last variant of our proposed block-coordinate framework for composite optimization, that involves direct minimization of the majorant $\Psi_k$ at iteration $k\in\mathbb{N}$ with respect to selected blocks.  %
The proof of convergence remains the same as for the two previous algorithm. For this algorithm only, we will assume that the majorants are strongly convex with respect to the selected set of variables to update at each iteration.
\noindent The strong convexity assumption on the majorant is formally stated below.
\begin{assumption} \label{ass:6}
    For every $k \in \mathbb{N}$, the majorant $\Psi_k$ is strongly convex with respect to the variables $(x_\ell)_{\ell \in J_i^k}$ for all $i\in\{0,\ldots,I_k-1\}$ with modulus $\mu_{k,i}>0$. Moreover, there exists $\underline{\mu}>0$ such that for all $k\in\mathbb{N}$, for all $i \in \{0,\ldots,I_k-1\}$, $\mu_{k,i} \geq \underline{\mu}$.
\end{assumption}

\paragraph{Our proposed Gauss-Seidel method.}
Start from a feasible $\mathbf{x}^0 \in \dom \Psi$, and alternate the minimization with respect to pre-defined group of blocks. The algorithm, named CO-Gauss-Seidel is detailed in Algorithm \ref{alg:gauss_seidel}.
\begin{algorithm}[t]
\caption{\textbf{CO-GS}}
\label{alg:gauss_seidel}
\begin{algorithmic}[1]
\For{$k = 0,1,\dots$}
  \State $\mathbf{\tilde{x}}_0^k \gets \mathbf{x}^k$
  \For{$i = 0,\ldots, I_k-1$}
    \State \textbf{Find} $(\tilde{x}_{i+1,\ell}^k)_{\ell\in J_i^k}$ such that
    \Statex \hspace{\algorithmicindent}%
    $(\tilde{x}_{i+1,\ell}^k)_{\ell\in J_i^k} \in \arg\min_{y_{J_i^k}} \ \Psi_k\!\big[\mathbf{\tilde{x}}_i^k\big]$
     and $\mathbf{\tilde{x}}_{i+1}^k \in \operatorname{dom}\Psi$
  \EndFor
  \State $\mathbf{x}^{k+1} \gets (\tilde{x}_{I_k,1}^{k},\ldots,\tilde{x}_{I_k,L}^{k})$
\EndFor
\end{algorithmic}
\end{algorithm}

\noindent $\argmin_{\ell \in J_i^k} \Psi_k[\mathbf{\tilde{x}}_i^k]$ denotes the minimization of the function $\Psi_k$ with respect to the variables $(x_\ell)_{\ell \in J_i^k}$, provided that the current iterate is $\mathbf{\tilde{x}}_i^k$, and the solution of each minimization lives in the domain of the function $\Psi$.
\begin{proposition}{\textbf{Sufficient decrease condition on the majorant function.}}
    \label{prop:gauss_cycle_majorant}
    Let $(\mathbf{x}^k)_{k\in\mathbb{N}}$ and, for every $k \in \mathbb{N}$, $(\mathbf{\tilde{x}}_i^k)_{0\leq i \leq I_k}$ be the iterates generated by Algorithm \ref{alg:gauss_seidel}. Let $k \in \mathbb{N}$ and $0 \leq i_1, < i_2 \leq I_k$. Under Assumptions \ref{ass:1}, \ref{ass:2}, and \ref{ass:6}, there exists $\eta>0$ such that
    \begin{equation} \label{eq:gauss_cycle_majorant}
        \Psi_k(\mathbf{\tilde{x}}_{i_2}^k) + \eta \sum_{i=i_1}^{i_2-1} \sum_{\ell \in J_i^k} \Vert \tilde{x}_{i+1,\ell}^k - \tilde{x}_{i,\ell}^k \Vert^2  \leq \Psi_k(\mathbf{\tilde{x}}_{i_1}^k)
    \end{equation}
    \end{proposition}
\begin{proof}
    We have for every $k \in \mathbb{N}$, $i \in \{0,\ldots,I_k\}$, and $\ell \in J_i^k$ that
    \begin{equation*}
        (\tilde{x}_{i+1,\ell}^k)_{\ell \in J_i^k} = \argmin_{\ell \in J_i^k} \Psi_k(\mathbf{\tilde{x}}_i^k).
    \end{equation*}
    Hence, using Assumption \ref{ass:6}, there exists $\underline{\mu}>0$ such that we have
    \begin{equation}
        \Psi_k(\mathbf{\tilde{x}}^k_{i+1}) + \frac{\underline{\mu}}{2}\sum_{\ell \in J_i^k} \Vert \tilde{x}_{i+1,\ell}^k - \tilde{x}_{i,\ell}^k \Vert^2 \leq \Psi_k(\mathbf{\tilde{x}}_{i}^k).
    \end{equation}
    Summing up from $i_1$ to $i_2-1$ and choosing $\eta = \underline{\mu}/2$, we obtain the desired result.
\end{proof}
A similar proposition to Proposition \ref{prop:descent_cycle} holds trivially for the iterates of Algorithm \ref{alg:gauss_seidel}, hence like for Algorithm \ref{alg:NEW-BC-FB} we do not restate it here.
\begin{proposition}{\textbf{Subgradient upper bound.}} \label{prop:subgradient_bound_gauss}
    Let $(\mathbf{x}^k)_{k\in\mathbb{N}}$ and, for every $k \in \mathbb{N}$, $(\mathbf{\tilde{x}}_i^k)_{0\leq i \leq I_k}$ be the iterates generated by Algorithm \ref{alg:gauss_seidel}. Let $k \in \mathbb{N}$ and $0 \leq i_1, < i_2 \leq I_k$. Under Assumptions \ref{ass:1}, \ref{ass:2}, and \ref{ass:6} we have for every $k \in \mathbb{N}$,
    \begin{equation}
        \vertiii{t(\mathbf{x}^{k+1})} \leq \kappa \sum_{i=0}^{I_k-1} \sum_{\ell \in J_i^k} \Vert \tilde{x}_{i+1,\ell}^k - \tilde{x}_{i,\ell}^k \Vert,
    \end{equation}
    where $\kappa>0$, and
    \begin{equation}
        \left\{ \begin{array}{l}
            t(\mathbf{x}^{k+1}) = \nabla f (\mathbf{x}^{k+1}) + v^{k+1} \in \partial \Psi(\mathbf{x}^{k+1}) \\
            v^{k+1}(\mathbf{x}^{k+1}) = \left( \lambda_{\ell,k+1} v_\ell^{k+1}(x_\ell^{k+1})\right)_{1\leq \ell \leq L}
        \end{array}\right.
    \end{equation}
    For every $\ell \in \{1,\ldots,L\}$, $v_\ell^{k+1}(x_\ell^{k+1}) =-\nabla_\ell f(\mathbf{\tilde{x}}_{k_\ell+1}^k) \in \partial  \psi_\ell(\tilde{x}_{k_\ell+1,\ell}^k)$, at the last iteration $k_\ell \in \{0,\ldots,I_k-1\}$ where the block $\ell$ has been updated.
    \end{proposition}
    \begin{proof}
        As in the proof of Proposition \ref{prop:subgradient_bound}, we have that $\partial \Psi(\mathbf{x}^{k+1}) = \nabla f (\mathbf{x}^{k+1}) + \partial g(\mathbf{x}^{k+1})$, with
         \begin{align}
    \partial g(\mathbf{x}^{k+1}) = \left( \lambda_{\ell,k+1} \partial \psi_\ell(\tilde{x}_{k_\ell+1,\ell}^{k})\right)_{1\leq \ell \leq L},
    \end{align}
    where $k_\ell$ is the last iteration where the block $\ell$ has been updated, and where $\lambda_{\ell,k+1} = (\phi'_\ell \circ \psi_\ell)(x_{\ell}^{k+1}) = (\phi'_\ell \circ \psi_\ell)(\tilde{x}_{k_\ell+1,\ell}^{k})$. Now at iteration $k$, we have for all $\ell \in J_{k_\ell,i}$ that
    \begin{equation}
        0 \in \nabla_\ell f(\mathbf{\tilde{x}}_{k_\ell+1}^k) + \partial \psi_\ell(\tilde{x}_{k_\ell+1,\ell}^k).
    \end{equation}
    Hence, there exists $v_\ell^{k+1}(x_{\ell}^{k+1}) \in \partial \psi_\ell(\tilde{x}_{k_\ell+1,\ell}^k)$ such that
    \begin{equation}
        v_\ell^{k+1}(x_{\ell}^{k+1}) = -\nabla_\ell f(\mathbf{\tilde{x}}_{k_\ell+1}^k) .
    \end{equation}
    Therefore, we have that
    \begin{align}
        \Vert \nabla_\ell f(\mathbf{x}^{k+1}) + v_\ell^{k+1}(x_\ell^{k+1}) \Vert & =  \Vert \nabla_\ell f(\mathbf{x}^{k+1}) - \nabla_\ell f(\mathbf{\tilde{x}}_{k_\ell+1}^k) \Vert
    \end{align}
    Now, we can use the block smoothness of $f$ to obtain
    \begin{align}
        \Vert \nabla_\ell f(\mathbf{x}^{k+1})-\nabla_\ell f(\mathbf{\tilde{x}}_{k_\ell+1}^k) \Vert & \leq \sum_{j=1}^L \beta_{\ell,j} \Vert x^{k+1}_j - \tilde{x}^{k}_{k_\ell+1,j} \Vert \nonumber \\
        & \leq \max_{\ell,j} \beta_{\ell,j} \sum_{i=0}^{I_k-1}  \sum_{j\in J_i^k}\Vert \tilde{x}_{i+1,j}^{k} - \tilde{x}_{i,j}^{k} \Vert.
    \end{align}
    Now set $\kappa = L \times \left(\max_{\ell,j} \beta_{\ell,j} \right)$ to obtain the desired bound on the norm of the subgradient:
    \begin{equation}
        \vertiii{t(\mathbf{x}^{k+1})} \leq \kappa \sum_{i=0}^{I_k-1}  \sum_{\ell \in J_i^k} \Vert \tilde{x}_{i+1,\ell}^k - \tilde{x}_{i,\ell}^k \Vert.
    \end{equation}
    \end{proof}
\noindent Again with Propositions \ref{prop:gauss_cycle_majorant} and \ref{prop:subgradient_bound_gauss}, Theorem \ref{th:convergence} holds for Algorithm \ref{alg:gauss_seidel}. Hence, we have convergence of the iterates to a critical point of $\Psi$.

\section{Numerical experiments: non-convex sparse matrix estimation}
\label{sec:graphical_LASSO}
In this section, we illustrate the computational advantages of our proposed framework to the reconstruction of sparse precision matrices. We illustrate for three classes of solvers for the non-convex Graphical Lasso problem that convergence to a solution can be obtained at minimal computational cost. To be consistent with the notations of the sparse precision matrix estimation literature, we denote the variable with $\boldsymbol{\Theta}$ instead of $\mathbf{x}$ in this section.
\paragraph{Graphical Lasso problem.}
Recall that we want to solve the following problem:
\begin{equation} \label{eq:nncvx_graphical_LASSO}
    \widehat{\boldsymbol{\Theta}} = \argmin_{\boldsymbol{\Theta} \in \mathbb{S}^d_{++}} \Psi(\boldsymbol{\Theta}) = - \log \det(\boldsymbol{\Theta}) + \mathrm{Tr}(S^\top\boldsymbol{\Theta}) + \sum_{i,j=1}^d \phi_{i,j}\left(|[\boldsymbol{\Theta}]^{(i,j)}|\right).
\end{equation}
where for all $i,j=1,\ldots,d$, $\phi_{i,j}$ follows Assumption \ref{ass:1}(iii) and (iv). We thus have $L=d^2$ blocks. This objective function does not directly fit the standard assumptions as it is not coercive and the data fidelity term $f:= -\log \det + \mathrm{Tr}(S^T \cdot)$ only has a locally Lipschitz continuous gradient \cite[Appendix A, Lemma 2]{rolfs2012iterative}. Hence, we use Assumption \ref{ass:bounded}, and suppose that the iterates generated by an algorithm solving \eqref{eq:nncvx_graphical_LASSO} will be bounded in eigenvalues above, and from below. In addition, note that the constraint $\boldsymbol{\Theta} \in \mathbb{S}^d_{++}$ is trivially satisfied if $\boldsymbol{\Theta}$ belongs to the domain of $\Psi$ due to the $\log \det$ term. Finally, under the boundedness assumption, strong convexity with respect to any subset of variables of the majorant function ensues.

\subsection{Chosen solvers for the Graphical Lasso problem.} As we said in the introduction, there exists numerous solvers to solve the Graphical Lasso problems. In the convex setting, there does not exist a silver bullet \cite{pouliquen2025quete}. Notably, the sparsity of the solution has a great effect on the performance of the algorithms. The fastest method for high sparsity levels is QUIC \cite{hsieh2013big,hsieh2014quic}. Hence, we want to provide a general acceleration framework, applicable to any solver of the Graphical Lasso problem \eqref{eq:nncvx_graphical_LASSO}. In order to complete this task, we use our proposed framework to accelerate three algorithms:
\begin{enumerate}
    \item Graphical ISTA \cite{rolfs2012iterative}, as Algorithm \ref{alg:MAJ-BC-FB},
    \item QUIC \cite{hsieh2013big,hsieh2014quic}, as Algorithm \ref{alg:NEW-BC-FB},
    \item Primal Graphical Lasso \cite{banerjee2008model,mazumder2012}, as Algorithm \ref{alg:gauss_seidel}.
\end{enumerate}
Under BC-CO, a block-coordinate version of Graphical ISTA is possible. However, several questions remain that are out of the scope of this paper, most of them being implementation issues. Our intent with these experiments is to show that with minimal changes (a few hyperparameters), existing solvers can solve efficiently non-convex formulation of the Graphical Lasso problem. These solvers have been refined to be very efficient in solving the convex Graphical Lasso problem, we simply show that this efficiency can be carried to the non-convex setting. 

\paragraph{Graphical ISTA as Algorithm \ref{alg:MAJ-BC-FB}.}
At each iteration, Graphical ISTA performs a proximal gradient step on $f + \mathbf{q}^g$ without variable metric \cite{rolfs2012iterative}. A line search is computed to ensure positive definiteness of the next iterate, but there exists a minimal safe step size \citep[Section 4]{rolfs2012iterative}. The stopping criterion is obtained by computing the duality gap, but this has no influence on our analysis.
\paragraph{QUIC as Algorithm \ref{alg:NEW-BC-FB}.} At each iteration, QUIC algorithm performs a proximal Newton update on a selected subset of coordinates, identified based on the norm of their respective gradient \cite{hsieh2013big,hsieh2014quic}. This Newton direction is itself computed using a coordinate descent algorithm. After that, a line search estimates the step size. It is a notable difference with our algorithm that needs to compute both the direction and the step size at the same time (otherwise the subgradient bound is lost). However in practice, the step is equal to $1$ after a finite number of iterations (i.e., when close to a solution), thus we expect our algorithm and QUIC to have equivalent behaviors. This was also shown theoretically in \citep{hsieh2013big}.

\paragraph{Primal-GLasso as Algorithm \ref{alg:gauss_seidel}.} For this algorithm, consider the following block decomposition of $\boldsymbol{\Theta}$:
\begin{equation}
    \boldsymbol{\Theta} = \left[\begin{array}{cc}
        \boldsymbol{\Theta}_{11} & \Theta_{12} \\
        \Theta_{12}^\top & \theta_{22}
    \end{array}\right],
\end{equation}
where $\boldsymbol{\Theta}_{11} \in \RR^{(d-1) \times (d-1)}$, $\Theta_{12} \in \RR^{(d-1) \times 1}$ and $\theta_{22} \in \RR$. At each iteration of the Primal-GLasso algorithm, the majorant $f+\mathbf{q}^g$ is minimized with respect to $\Theta_{12}, \Theta_{12}^\top$, and $\theta_{22}$ keeping $\boldsymbol{\Theta}_{11}$ fixed (then a simple permutation is done at the next iteration). This partial minimization is done by coordinate descent, and is guaranteed to yield positive definite updates \cite{mazumder2012,pouliquen2025quete,pouliquen2024schur}. 

\subsection{Experimental setting} The experimental setting is similar to that of \citep{pouliquen2025quete}, whose authors created a benchmark for the Graphical Lasso problem in convex and non-convex settings.
\paragraph{Dataset.} We generate synthetic data  by creating a sparse precision matrix $\boldsymbol{\Theta}_{\mathrm{true}} \in \RR^{75 \times 75}$ using \texttt{sklearn.datasets.make$\_$sparse$\_$spd$\_$matrix} \citep{pedregosa2011scikit}, onto which is added $0.1 \Id$ to slightly improve conditioning, following \citep{pouliquen2025quete}. We set the sparsity level to $90\%$, i.e., only $10\%$ of the entries are non-zero. Then, we generate $P = 1000$ independent samples of a centered Gaussian distribution $X^{(i)} \sim \mathcal{N}\left(0,\boldsymbol{\Theta}_{\mathrm{true}}^{-1}\right)$. These samples are used to compute the empirical covariance matrix $S = \frac{1}{P}\sum_{i=1}^P X^{(i)} (X^{(i)})^\top$. All the proposed solvers are initialized with this information to solve Problem \eqref{eq:nncvx_graphical_LASSO}. 
\paragraph{Non-convex penalties.} We consider the following non-convex penalties: $\phi = \log(\cdot + \epsilon)$ (log sum), $\phi = \sqrt{\cdot + \epsilon}$ ($\ell_{0.5}$), and $\phi = \mathrm{MCP}$ \citep{zhang2010nearly}.
The subsequent majorants of $g$ are then defined with the closed form expression of Section \ref{sec:framework}. The weights are then:
for $\log$-sum, $\lambda_{\ell,k} = \gamma_\ell (|\theta_\ell^k|+\epsilon)^{-1}$. For $\ell_{0.5}$, $\lambda_{\ell,k} = \gamma_l\left(2\sqrt{|\theta_\ell^k|+\epsilon}\right)^{-1}$. For MCP, $\lambda_{\ell,k} = \max\left\{0,\gamma_\ell - |\theta_\ell^k|/\epsilon\right\}$. We detail the computation for the log and the MCP penalties in Appendix \ref{app:reweighting}.
\paragraph{Reconstruction metrics.} We cannot expect our algorithms to converge to the same solution due to the non-convexity of the problem. Therefore we estimate the quality of each solution by computing the F1-score between the binary matrices highlighting the support of $\boldsymbol{\Theta}_{\mathrm{true}}$ and $\boldsymbol{\widehat{\Theta}}$, and by computing the normalized mean squared error $NMSE = \frac{\vertiii{\boldsymbol{\widehat{\Theta}}-\boldsymbol{\Theta}_{\mathrm{true}}}^2}{\vertiii{\boldsymbol{\Theta}_{\mathrm{true}}}^2}$. These two metrics provide complementary information about the quality of estimation: the F1-score tells us how well the support of the true precision matrix is estimated.

For each penalty, we identify with a grid search the value of the optimal regularization parameter $\gamma$ (at initialization) providing the best reconstruction performance, i.e. the maximum F1-score, and the minimum normalized mean squared error. The grid search is computed with $20$ values of $\lambda$.

\paragraph{Experimental results.}  

For each algorithm, we compute $20$ reweighting operations, i.e., updates of the components' penalization. What changes is the number of iterations in between these reweighting operations. Thus we go from $20$ total iterations (i.e., $1$ iteration per reweighting), to $2000$ (i.e., $100$ per reweighting which is the standard configuration \cite{pouliquen2025quete}).

We present the results separately for the three algorithms in Figures \ref{fig:gista0.9} (Graphical ISTA), \ref{fig:quic0.9} (QUIC), and \ref{fig:pglasso0.9} (P-GLasso). For each algorithm we subdivide the figure in four plots:
\begin{itemize}
    \item Top left corner: maximum F$1$-score w.r.t. the number of iterations between reweighting operations. 
    \item Bottom left corner: maximum F$1$-score w.r.t. the total number of iterations.
    \item Top right corner: minimum normalized MSE w.r.t. the number of iterations between reweighting operations.
    \item Bottom right corner: minimum normalized MSE w.r.t. the total number of iterations.
\end{itemize}
In blue, on each plot, we display the given metric for the standard convex $\ell_1$ penalty as a comparison. Every dot represents the same run. 
\begin{figure}[ht]
\centering
\begin{minipage}{0.45\textwidth}
    \centering
    \includegraphics[width=\textwidth]{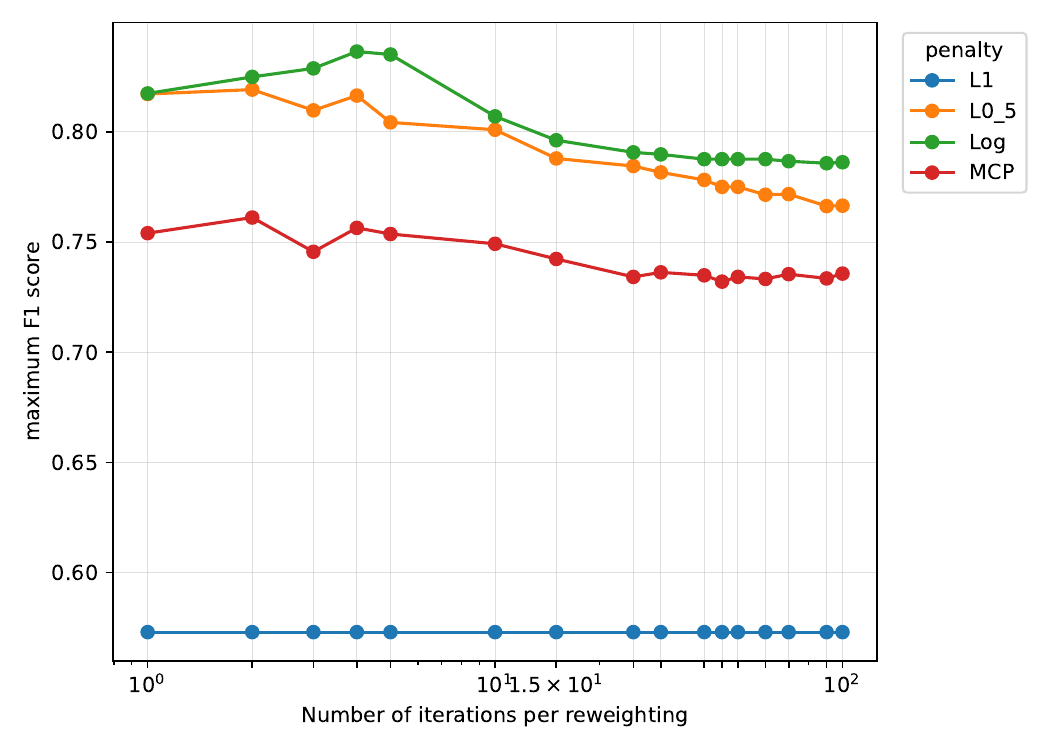}\\
    \includegraphics[width=\textwidth]{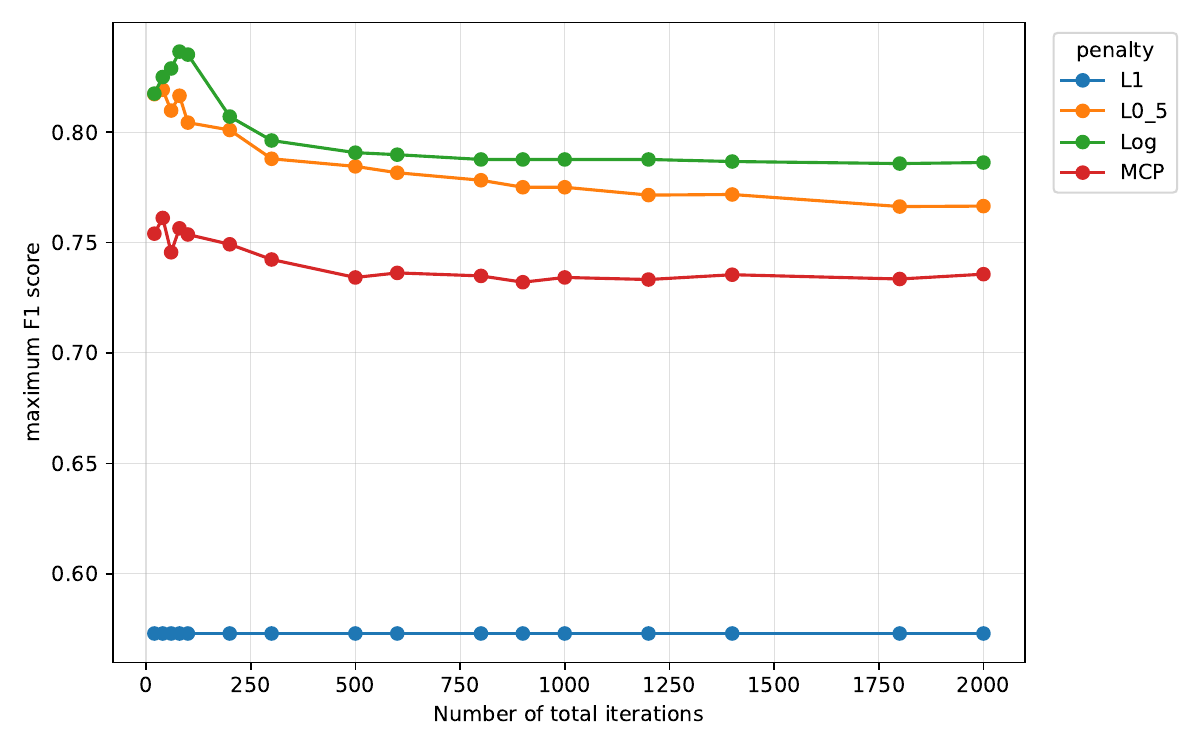}
\end{minipage}\hspace{1em}
\begin{minipage}{0.45\textwidth}
    \centering
    \includegraphics[width=\textwidth]{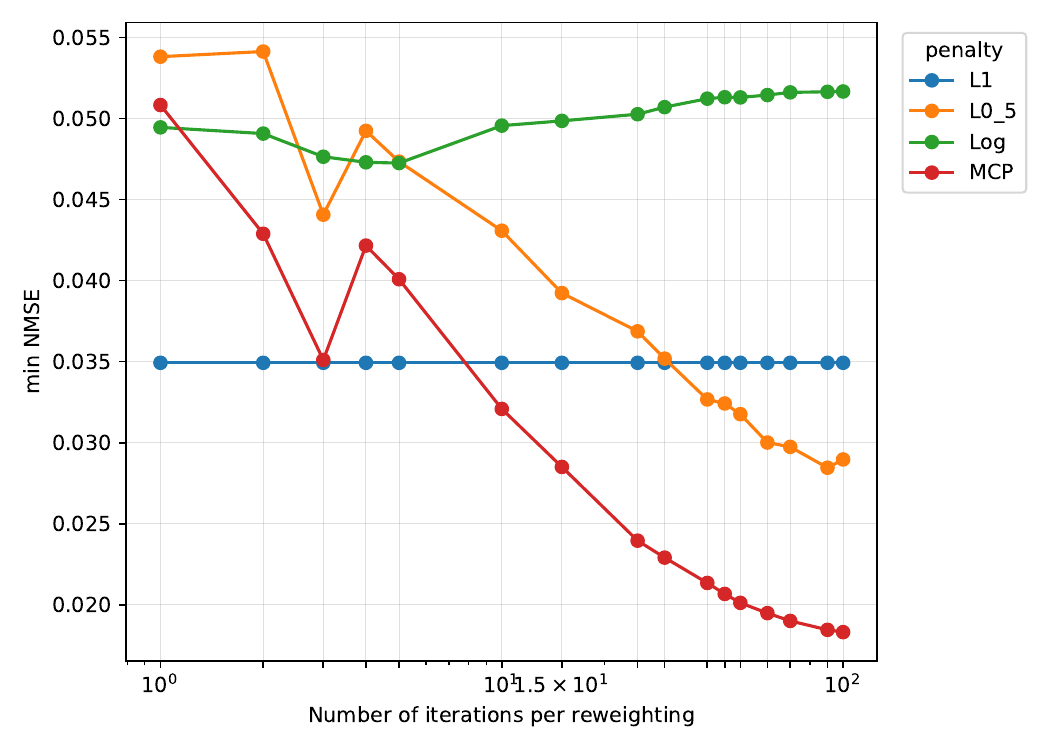}\\
    \includegraphics[width=\textwidth]{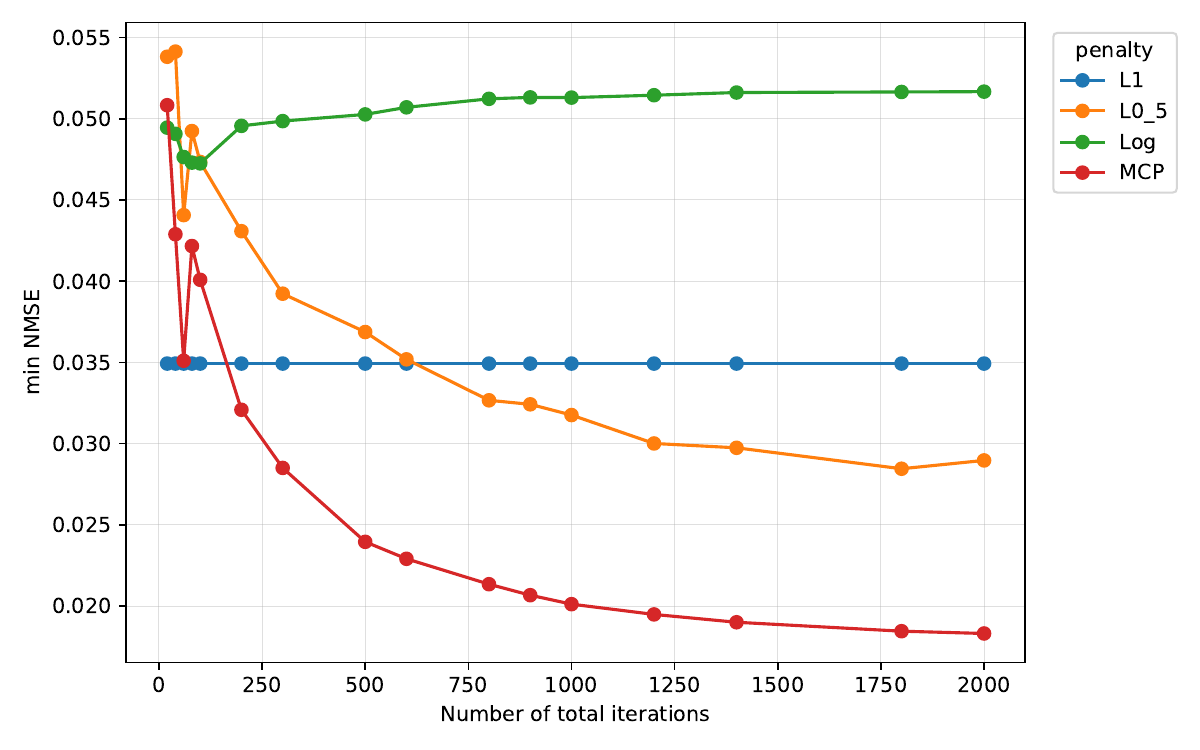}
\end{minipage}
\caption{Graphical-ISTA reconstruction for a sparsity level of $90\%$. Top left: maximum F$1$-score w.r.t. number of iterations between reweightings (log-scale). Bottom left: maximum F$1$-score w.r.t. total number of iterations. Top right: minimum NMSE w.r.t. number of iterations between reweightings (log-scale). Bottom right: minimum NMSE  w.r.t. total number of iterations. } \label{fig:gista0.9} 
\end{figure}
We can immediately observe that support identification does not improve with the number of iterations between two reweightings for all three algorithms above $10$ iterations per reweighting (which is far from a complete minimization). Moreover, for QUIC and P-GLasso there seems to be a spot between $10$ and $15$ iterations per reweighting where the normalized mean squared error reaches a minimum value. This means that the total iterations budget can be divided $8$ to $10$ fold, while maintaining the same estimation quality as the standard reweighting approach. For Graphical ISTA the picture is less clear, the total iterations budget can only be halved. A possible explanation is that this algorithm is not a block-coordinate one, hence in a high sparsity setting like here the other two algorithms can exploit this information and thus each iteration provides more progress.
\begin{figure}[t]
\centering
\begin{minipage}{0.45\textwidth}
    \centering
    \includegraphics[width=\textwidth]{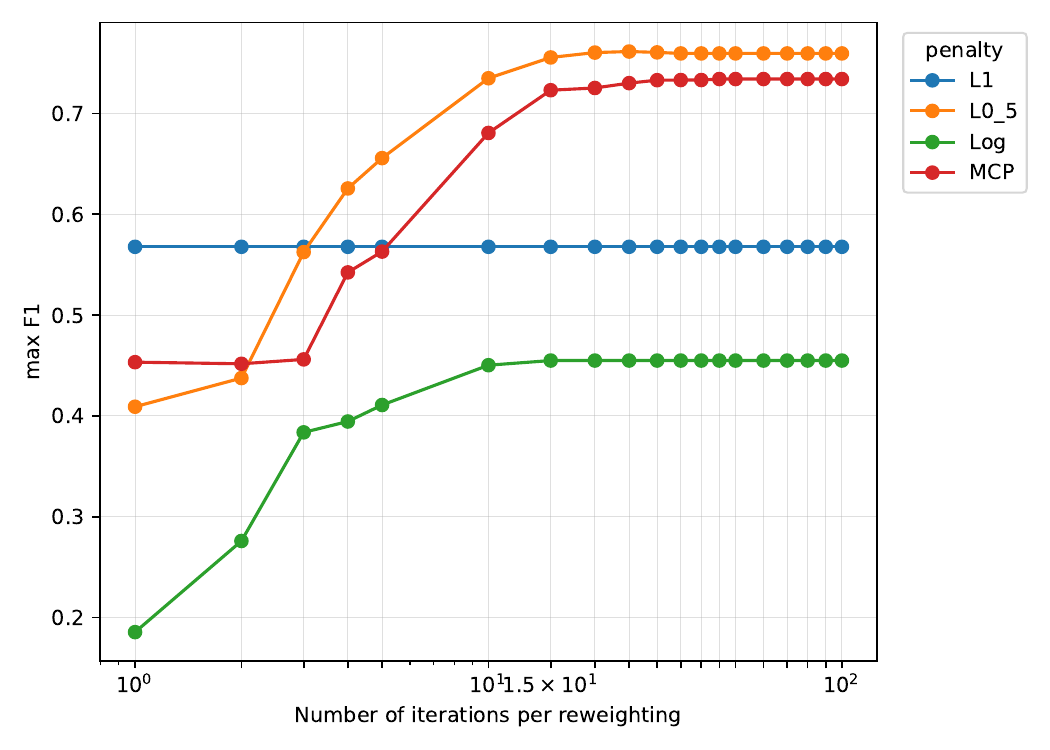}\\
    \includegraphics[width=\textwidth]{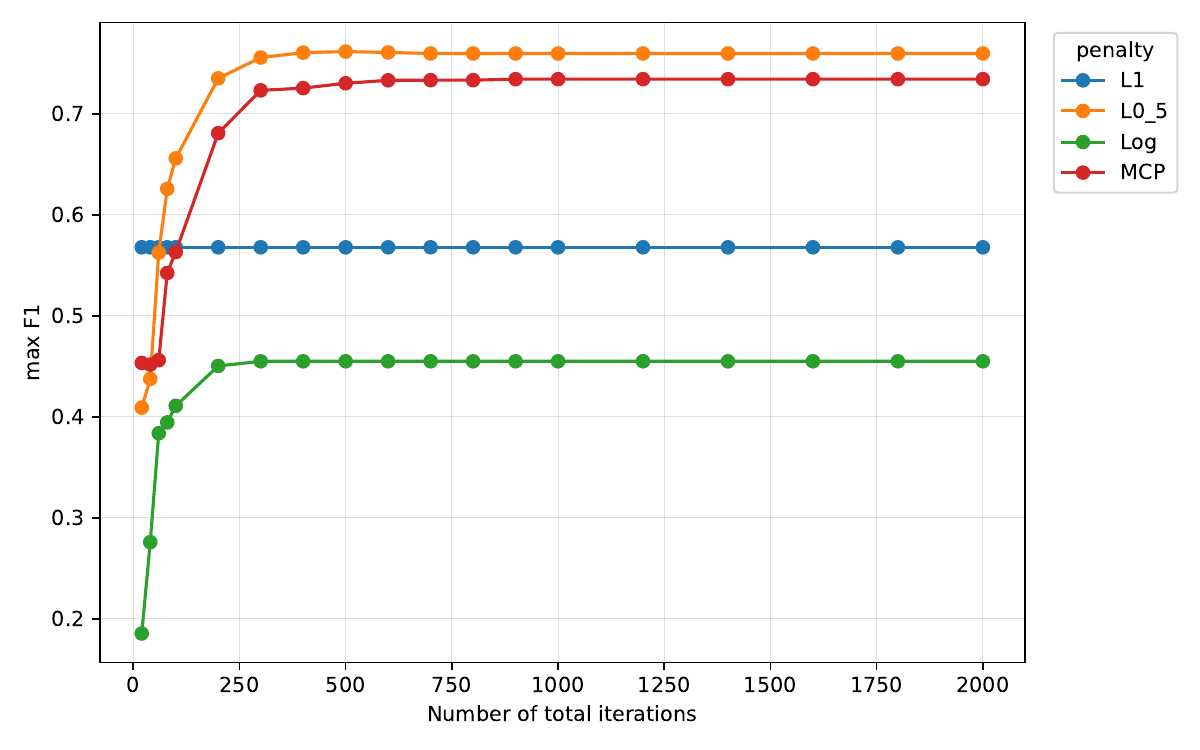}
\end{minipage}\hspace{1em}
\begin{minipage}{0.45\textwidth}
    \centering
    \includegraphics[width=\textwidth]{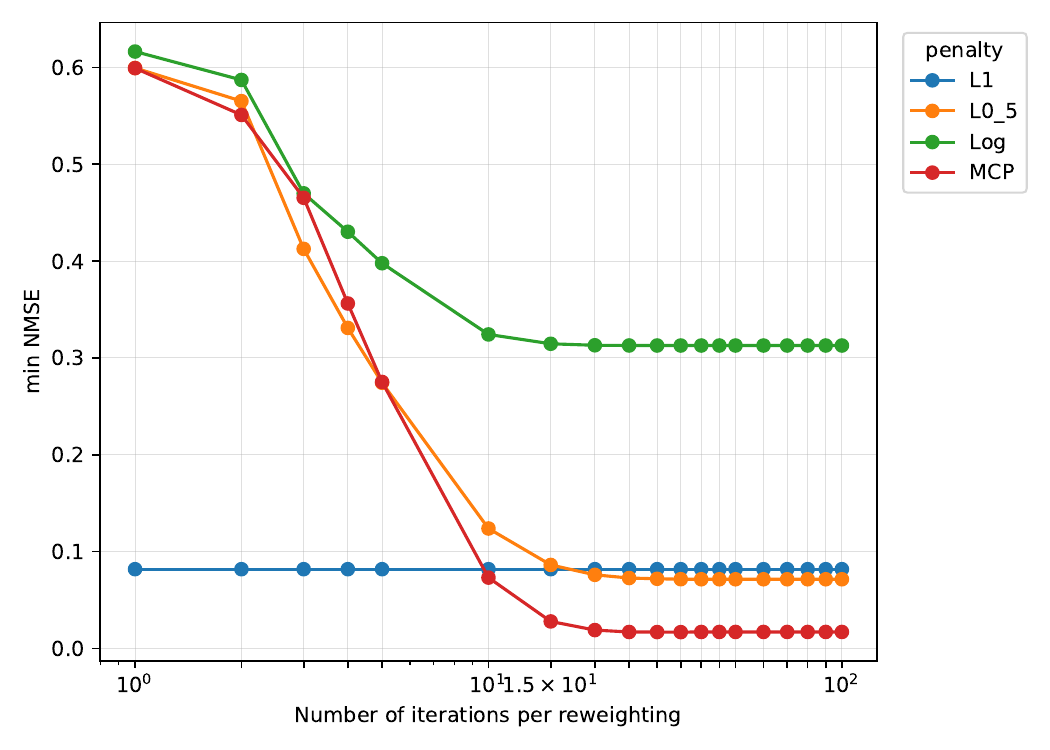}\\
    \includegraphics[width=\textwidth]{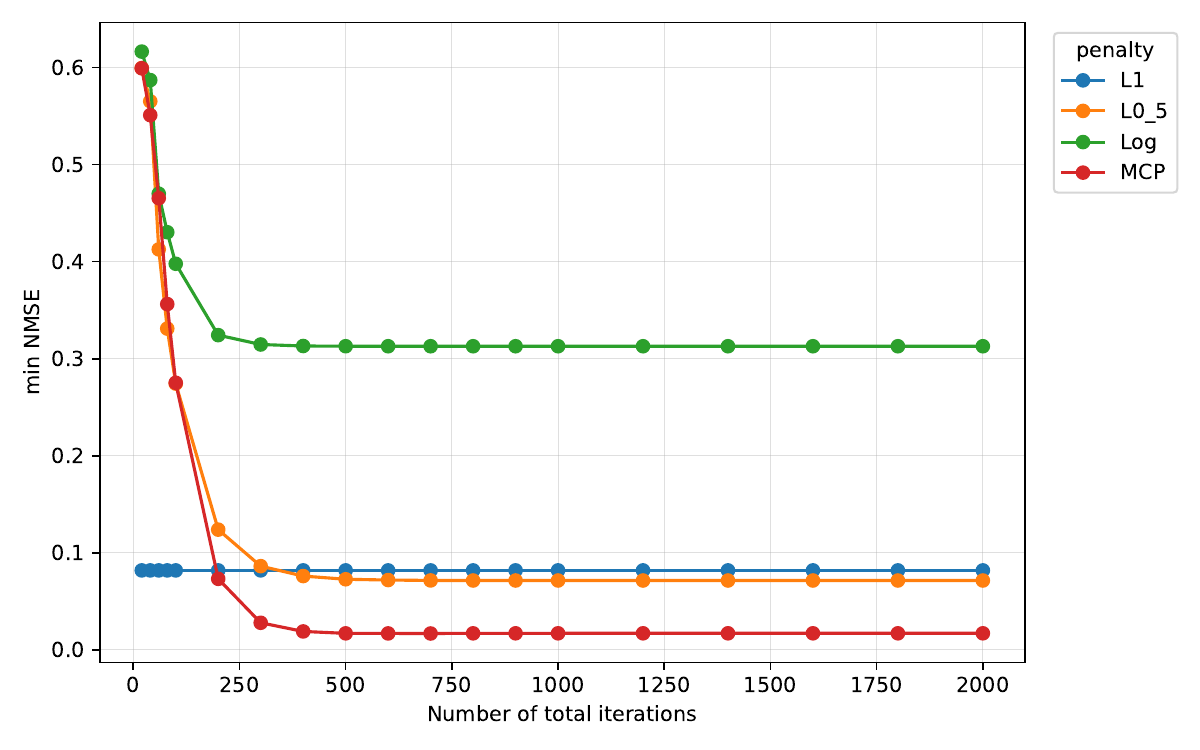}
\end{minipage}
\caption{QUIC reconstruction for a sparsity level of $90\%$. Top left: maximum F$1$-score w.r.t. number of iterations between reweightings (log-scale). Bottom left: maximum F$1$-score w.r.t. total number of iterations. Top right: minimum NMSE w.r.t. number of iterations between reweightings (log-scale). Bottom right: minimum NMSE  w.r.t. total number of iterations.} \label{fig:quic0.9} 
\end{figure}

\begin{figure}[t]
\centering
\begin{minipage}{0.45\textwidth}
    \centering
    \includegraphics[width=\textwidth]{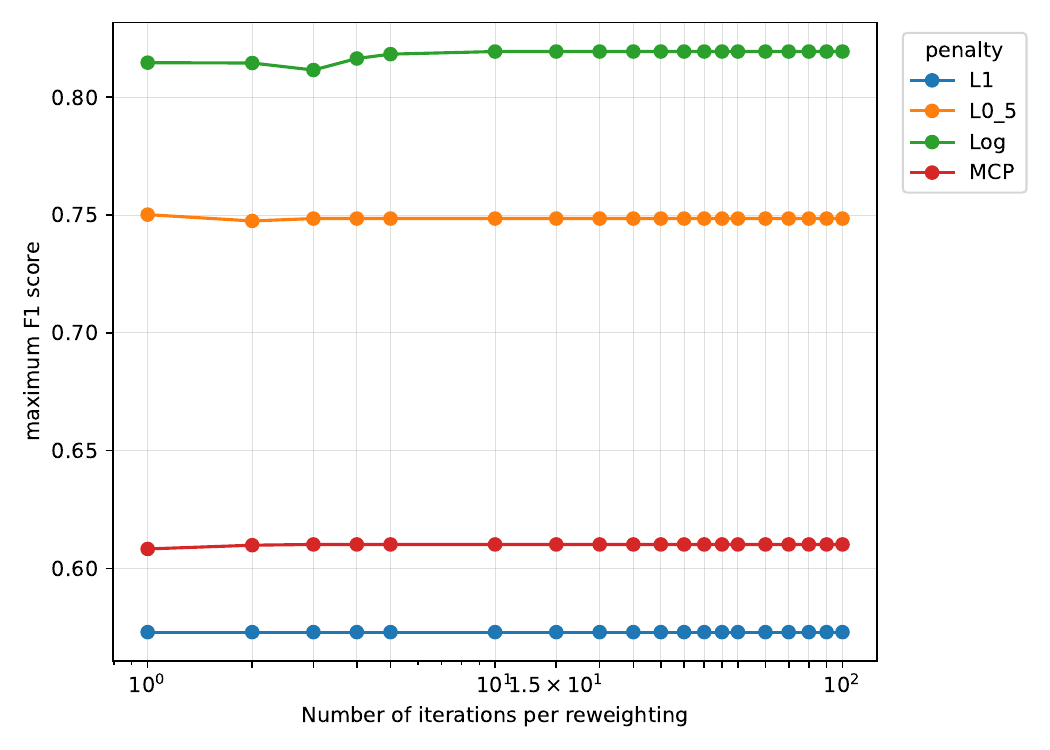}\\
    \includegraphics[width=\textwidth]{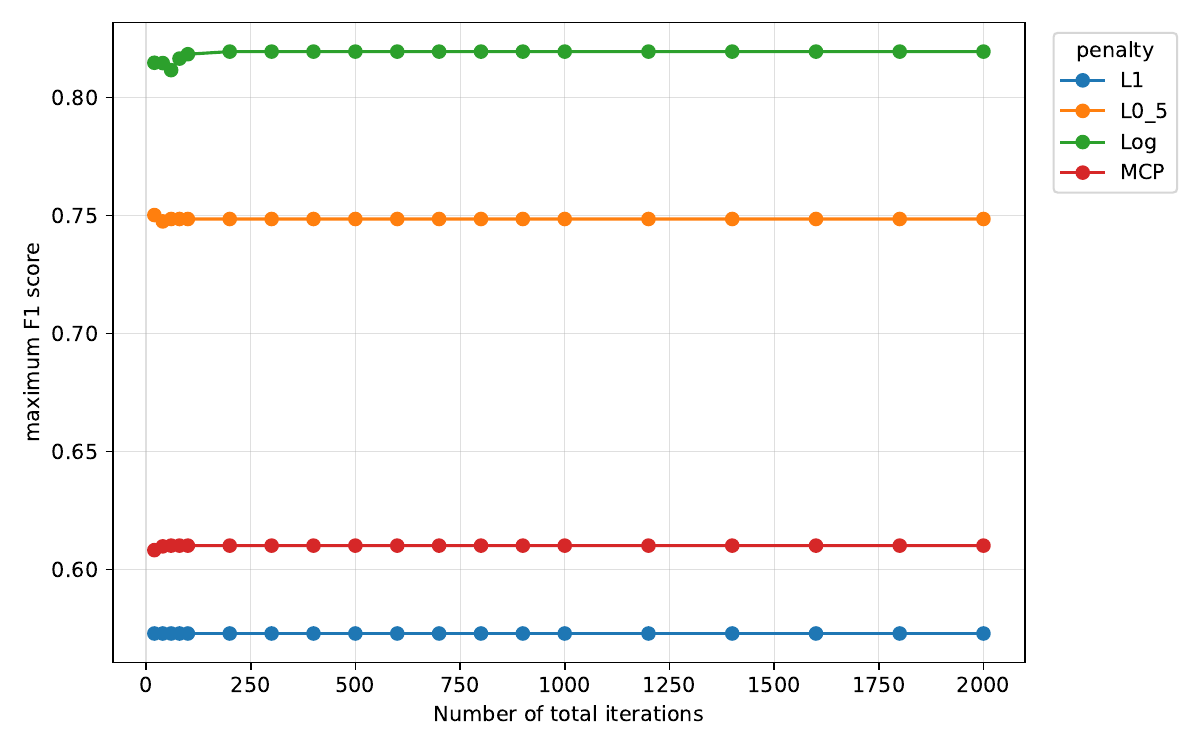}
\end{minipage}\hspace{1em}
\begin{minipage}{0.45\textwidth}
    \centering
    \includegraphics[width=\textwidth]{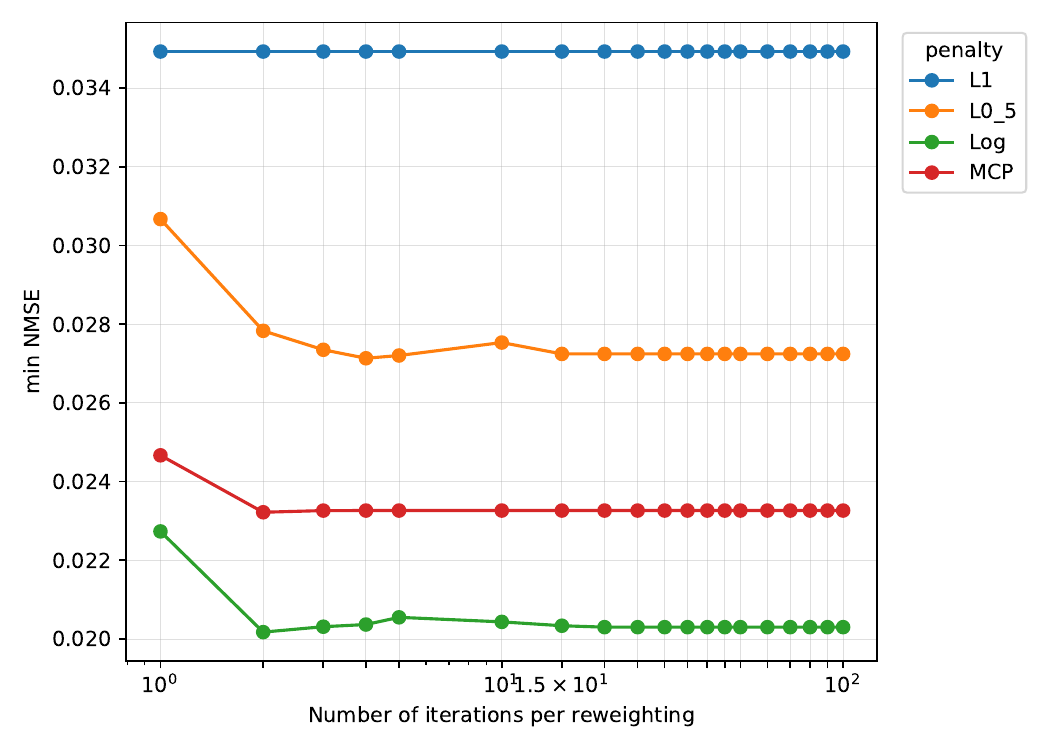}\\
    \includegraphics[width=\textwidth]{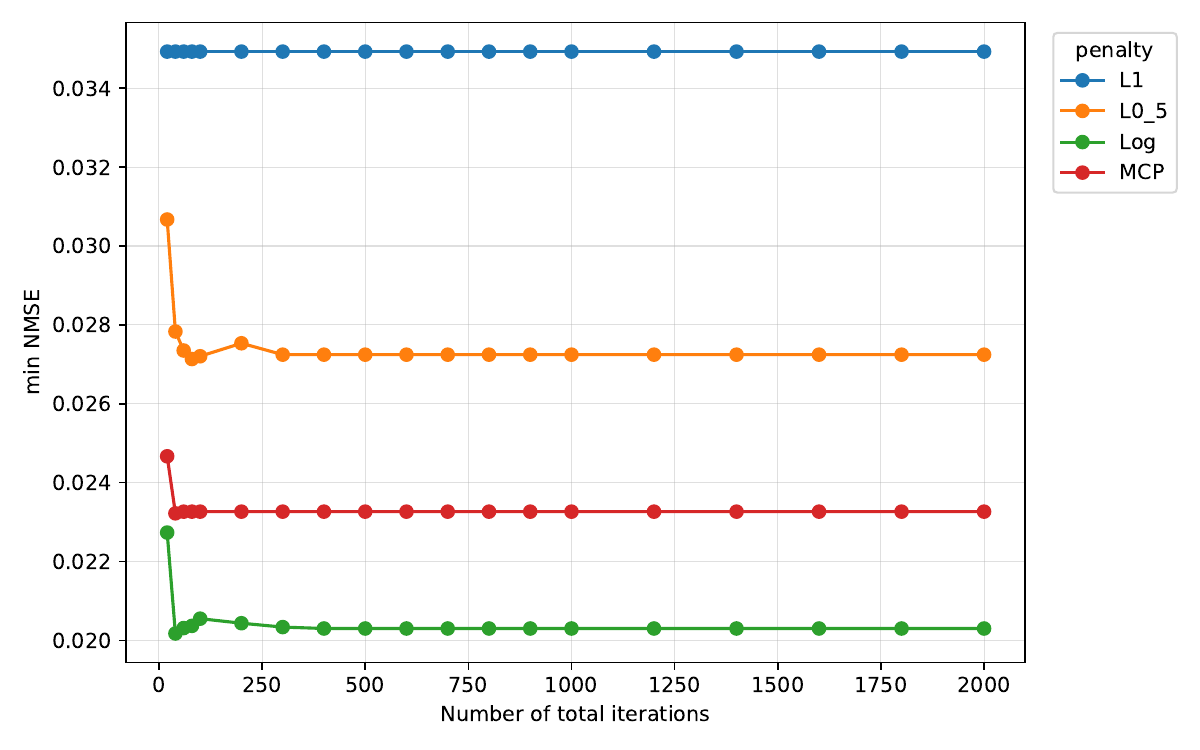}
\end{minipage}
\caption{P-GLasso reconstruction for a sparsity level of $90\%$. Top left: maximum F$1$-score w.r.t. number of iterations between reweightings (log-scale). Bottom left: maximum F$1$-score w.r.t. total number of iterations. Top right: minimum NMSE w.r.t. number of iterations between reweightings (log-scale). Bottom right: minimum NMSE  w.r.t. total number of iterations.} \label{fig:pglasso0.9} 

\end{figure}
The reader can also find in Appendix \ref{app:numerics} additional experimental results with a sparsity level of $75\%$. In this context, QUIC performance degrades which is expected given that it relies heavily on the sparsity of the solution \cite{hsieh2014quic,hsieh2013big}.

\section{Conclusion}
In this paper, we present a new framework for the minimization of non-convex composite functions, with a block-coordinate descent perspective. This framework is sufficiently general for us to derive three algorithms which tackle this optimization problem in different fashion, and thus with different strengths: a variable metric block-coordinate forward-backward algorithm, a block-coordinate proximal Newton algorithm, and a Gauss--Seidel algorithm.

We applied these algorithms on the non-convex Graphical Lasso problem, for which three versions of these algorithms exist: Graphical ISTA, QUIC and P-GLasso. We showed on a simple experiment that replacing the complete minimization at each reweighting step by a descent condition can yield impressive gain in terms of total number of iterations.

These experimental results could be further improved by developing a block-coordinate version of the Graphical ISTA algorithm, which is as of now the slowest algorithm of the three. Furthermore, this framework could be applied to different problems where reweighting is a relevant method to obtain better estimation performance like in radio-interferometric imaging \cite{repetti2020forward}.

Venues for theoretical improvements are numerous, for instance the relaxation of the Lipschitz continuity (or local Lipschitz continuity) and the explicit incorporation of line searches \cite{bonettini2018block}. Or the modelling of inexactness of the computation of proximal operators \cite{bonettini2020convergence}, which is not easily available here even though inexactness is possible.

\section*{Acknowledgments}
We thank Can Pouliquen for his great help on understanding the Graphical Lasso problem, and for the code he produced to benchmark algorithms against this problem.

\section*{Funding}
The author was supported in the beginning of this project by Fondation Simone et Cino Del Duca - Institut de France, and for the rest by MAD (ANR-24-CE23-1529) project of the French National Agency for Research (ANR).

\bibliographystyle{plain}
\bibliography{references}


\begin{appendices}
\section{Additional results} \label{app:proofs}
\subsection{Proof of Theorem \ref{th:convergence}}
The proof of the convergence theorem is identical to the one of \citep[Theorem 3]{lauga2025BCD}, up to some constants.
\begin{theorem} Suppose that Assumptions \ref{ass:1}, \ref{ass:2}, \ref{ass:3}, \ref{ass:4}  and \ref{ass:5} hold.
    Let $\{\mathbf{x}^k\}_{k \in \mathbb{N}}$ be a sequence generated by Algorithm \ref{alg:MAJ-BC-FB}. The following hold
    \begin{enumerate}[(ii)]
        \item $\{\mathbf{x}^k\}_{k \in \mathbb{N}}$ has finite length, i.e., $\sum_{k=0}^{+\infty} \vertiii{\mathbf{x}^{k+1}-\mathbf{x}^k} < + \infty$.
        \item $\{\mathbf{x}^k\}_{k \in \mathbb{N}}$ converges to a critical point $\widehat{\mathbf{x}}$ of $\Psi$.
    \end{enumerate}
\end{theorem}
\begin{proof}
\begin{enumerate}[(ii)]
  \item Since the sequence $(\mathbf{x}^k)_{k\in\mathbb{N}}$ is bounded, there exists a sub-sequence that converges to $\widehat{\mathbf{x}}$. As $\{\Psi(\mathbf{x}^k)\}_{k \in \mathbb{N}}$ is a non-increasing sequence, and as the limit points set $\lp(\mathbf{x}^0)$ is such that $\lim_{k \rightarrow \infty} \mathrm{dist} (\mathbf{x}^k,\lp(\mathbf{x}^0)) = 0$ (Lemma \ref{lm:limit_point_set}, point (ii)), there exist $k_0 \in \mathbb{N},\varepsilon >0, \eta >0$ such that for all $k > k_0$, $\mathbf{x}^k$ belongs to:
  \begin{equation*}
      \left\{\mathbf{x} \in \RR^d: \mathrm{dist} (\mathbf{x},\lp(\mathbf{x}^0)) < \varepsilon \right\} \cap [\Psi(\widehat{\mathbf{x}}) < \Psi(\mathbf{x}) < \Psi(\widehat{\mathbf{x}}) + \eta].
  \end{equation*}
  Using now that $\lp(\mathbf{x}^0)$ is nonempty and compact, and that $\Psi$ is constant on it (Lemma \ref{lm:limit_point_set}, points (ii) and (iv)), one can apply Lemma \ref{lm:unif_KL} so that for any $k>k_0$:
      \begin{equation*}
          \varphi'(\Psi(\mathbf{x}^k)-\Psi(\widehat{\mathbf{x}})) \mathrm{dist}(0,\partial \Psi (\mathbf{x}^k)) \geq 1.
      \end{equation*}
  Now, at least one element of $\partial \Psi (\mathbf{x}^k)$ has its norm bounded (Proposition \ref{prop:subgradient_bound}), thus $\mathrm{dist}(0,\partial \Psi (\mathbf{x}^k))$ is necessarily less than or equal to this bound. We have
  \begin{equation}
      \mathrm{dist}(0,\partial \Psi (\mathbf{x}^k)) \leq \kappa \sum_{i=0}^{I_{k-1}-1}  \sum_{\ell \in J_{k-1,i}} \Vert \tilde{x}_{i+1,\ell}^{k-1} - \tilde{x}_{i,\ell}^{k-1} \Vert,
  \end{equation}
  Denote by $\kappa D_{k-1}$ the right-hand side of the previous inequality. We have:
  \begin{align*}
        \varphi'(\Psi(\mathbf{x}^k)-\Psi(\widehat{\mathbf{x}})) \kappa D_{k-1} \geq 1 \implies  \varphi'(\Psi(\mathbf{x}^k)-\Psi(\widehat{\mathbf{x}}))  \geq \kappa^{-1}D_{k-1}^{-1}.
  \end{align*}
  The concavity of $\varphi$ yields that:
      \begin{equation}
          \varphi(\Psi(\mathbf{x}^k)-\Psi(\widehat{\mathbf{x}})) - \varphi(\Psi(\mathbf{x}^{k+1})-\Psi(\widehat{\mathbf{x}})) \geq \varphi'(\Psi(\mathbf{x}^k)-\Psi(\widehat{\mathbf{x}})) \left(\Psi(\mathbf{x}^k)-\Psi(\mathbf{x}^{k+1})\right).
      \end{equation}
  Now denote by
  \begin{equation}
      C_k = \sum_{i=0}^{I_{k}-1}  \sum_{\ell \in J_i^k} \Vert \tilde{x}_{i+1,\ell}^{k} - \tilde{x}_{i,\ell}^{k} \Vert^2.
  \end{equation}
  From Proposition \ref{prop:descent_cycle_majorant}, we have that $\Psi(\mathbf{x}^k)-\Psi(\mathbf{x}^{k+1}) \geq \eta C_k$. We can then write that:
  \begin{align*}
      \frac{\eta}{\kappa}\frac{C_k}{D_{k-1}} \leq \varphi(\Psi(\mathbf{x}^k)-\Psi(\widehat{\mathbf{x}})) - \varphi(\Psi(\mathbf{x}^{k+1})-\Psi(\widehat{\mathbf{x}})),
  \end{align*}
  which we can rewrite with  $\Delta_k =\varphi(\Psi(\mathbf{x}^k)-\Psi(\widehat{\mathbf{x}})) - \varphi(\Psi(\mathbf{x}^{k+1})-\Psi(\widehat{\mathbf{x}})) $and hence obtain $C_k  \leq \frac{\kappa}{\eta}D_{k-1} \Delta_k$.
  Moreover, we have 
  \begin{align*}
      C_k  = \sum_{\substack{0 \leq i <I_k \\ 1 \leq \ell \leq L}} \Vert \tilde{x}_{i+1,\ell}^{k} - \tilde{x}_{i,\ell}^{k} \Vert^2 \text{ and } D_{k} =\sum_{\substack{0 \leq i < I_{k} \\ 1 \leq \ell \leq L}} \Vert \tilde{x}_{i+1,\ell}^{k} - \tilde{x}_{i,\ell}^{k} \Vert
  \end{align*}
  The $1$-trick of Cauchy's inequality then yields:
  \begin{equation*}
      D_k \leq \sqrt{I_k \times L} \sqrt{C_k},
  \end{equation*}
  which, combined with the fact that $2\sqrt{ab} \leq a + b$  for $a,b\geq 0$, gives:
  \begin{align*}
      2 D_k & \leq 2 \sqrt{I_k \times L} \sqrt{C_k}
      \leq 2 \sqrt{L I_k \frac{\kappa}{\eta}D_{k-1} \Delta_k} 
      \leq \frac{I_k L \kappa}{\eta}\Delta_k + D_{k-1}.
  \end{align*}
  Summing up this inequality from $k=k_0+1$ to $k=K$, we obtain:
  \begin{align*}
      2 \sum_{k=k_0+1}^K D_k & \leq \sum_{k=k_0+1}^K \left(\frac{L I_k \kappa}{\eta}\Delta_k + D_{k-1}\right)\\
      & \leq \sum_{k=k_0+1}^K D_{k-1} + \frac{L \bar{I} \kappa}{\eta}\sum_{k=k_0+1}^K \Delta_k \\
      & \leq D_{k_0} +\sum_{k=k_0+1}^K D_{k} +  \frac{L \bar{I} \kappa}{\eta}\sum_{k=k_0+1}^K \Delta_k.
  \end{align*}
  Hence,
  \begin{align*}
      \sum_{k=k_0+1}^K D_k & \leq  D_{k_0} + \frac{L \bar{I} \kappa}{\eta}\sum_{k=k_0+1}^K \Delta_k  \\
      & \leq D_{k_0} + \frac{L \bar{I} \kappa}{\eta}\left(\varphi(\Psi(\mathbf{x}^{k_0+1})-\Psi(\widehat{\mathbf{x}})) - \varphi(\Psi(\mathbf{x}^{K+1})-\Psi(\widehat{\mathbf{x}})) \right) \\
      & \leq  D_{k_0} + \frac{L \bar{I} \kappa}{\eta}\varphi(\Psi(\mathbf{x}^{k_0+1})-\Psi(\widehat{\mathbf{x}})).
  \end{align*}
  The last inequality comes from the positiveness of $\varphi$ (Definition \ref{def:desingularizing}). Since $\vertiii{\mathbf{x}^{k+1}-\mathbf{x}^k} \leq D_k$, we have that $\{\mathbf{x}^k\}_{k \in \mathbb{N}}$ has finite length.
  \item Since $\{\mathbf{x}^k\}_{k \in \mathbb{N}}$ has finite length, it is a Cauchy sequence, hence it converges to a point $\widehat{\mathbf{x}}$ in $\lp(\mathbf{x}^0)$ which is a critical point of $\Psi$ (Lemma \ref{lm:limit_point_set}, point (i)).
\end{enumerate}
\end{proof}

\section{Reweighting of $\ell_1$ penalties}
\label{app:reweighting}
In this section, we present the construction of a majorant function for two non-convex penalties amenable to this type of optimization. For the rest of this section, we have for every $\ell\in \{1,\ldots,L\}$, $\psi_\ell := |\cdot|$.
\paragraph{Log-sum penalty \citep{prater2022proximity}.} Define for every $\ell\in\{1,\ldots,L\}$, the function $\phi_\ell(u) = \gamma_\ell \log(u+\epsilon)$ for every $u\in[0,+\infty]$ where $\gamma_\ell>0$ and $\epsilon>0$ are fixed parameters. The function $g$ implementing a log-sum prior is then defined as:
\begin{equation}
    (\forall \mathbf{x} \in \Hi), \quad g(\mathbf{x}) = \sum_{\ell=1}^L \gamma_\ell \log(|x_\ell|+\epsilon).
\end{equation}
and for every $k \in \mathbb{N}$, the majorant $q(\cdot, \mathbf{x}^k)$ of $g$ is given by:
\begin{equation}
    (\forall \mathbf{x} \in \Hi), \quad q(\mathbf{x},\mathbf{x}^k) = \sum_{\ell=1}^L \gamma_\ell \log(|x_\ell^k|+\epsilon) + \lambda_{\ell,k}(|x_\ell|-|x_\ell^k|),
\end{equation}
where for every $\ell\in\{1,\ldots,L\}$, $\lambda_{\ell,k} = \gamma_\ell (|x_\ell^k|+\epsilon)^{-1}$.
It leads to the following reweighting operation for all $k$ iterations:
\begin{equation}
    \Lambda_k = \textrm{Diag }\left( \left( \frac{\gamma_\ell}{|x_\ell^k|+\epsilon}\right)_{1\leq \ell \leq L} \right)
\end{equation}

\paragraph{Minimax concave penalty \citep{shen2019structured}.} Define for every $\ell\in\{1,\ldots,L\}$, the function $\phi_\ell(u) = \gamma_\ell u - \mathrm{env}_{\epsilon \gamma_\ell \psi_\ell}(u)$ for every $u\in[0,+\infty]$ where $\gamma_\ell>0$, $\epsilon>1$, and where $\mathrm{env}_{\theta \psi_\ell}$ is the Moreau envelope of $\gamma_\ell \psi_\ell$. Note that here $\epsilon$ does not play the same role as for the log-sum penalty. The function $g$ implementing the minimax concave penalty (MCP) is then defined as:
\begin{equation}
    (\forall \mathbf{x} \in \Hi), \quad g(\mathbf{x}) = \sum_{\ell=1}^L  \gamma_\ell |x_\ell| - \mathrm{env}_{\epsilon \gamma_\ell |\cdot|}(|x_\ell|).
\end{equation}
and for every $k \in \mathbb{N}$, the majorant $q(\cdot, \mathbf{x}^k)$ of $g$ is given by:
\begin{equation}
    (\forall \mathbf{x} \in \Hi), \quad q(\mathbf{x},\mathbf{x}^k) = \sum_{\ell=1}^L \gamma_\ell |x_\ell^k| - \mathrm{env}_{\epsilon \gamma_\ell |\cdot|}(|x_\ell^k|) + \lambda_{\ell,k}(|x_\ell|-|x_\ell^k|).
\end{equation}
where for every $\ell\in\{1,\ldots,L\}$, $\lambda_{\ell,k} = \max\{0,\gamma_\ell- |x_\ell^k|/\epsilon\}$. The reweighting operation for all $k$ iterations is then given by:
\begin{equation}
    \Lambda_k = \textrm{Diag }\left( \left( \max\{0,\gamma_\ell - \frac{|x_\ell^k|}{\epsilon}\}\right)_{1\leq \ell \leq L} \right).
\end{equation}

\subsection{Additional numerical experiments} \label{app:numerics} In this section, we provide additional numerical experiments on the non-convex Graphical Lasso when the sparsity level of the ground truth precision matrix is $75\%$ (i.e., a $25\%$ of non-zero components). One can directly see that to obtain the best reconstruction possible, the number of iteration per reweighting increases.

\begin{figure}[ht]
\centering
\begin{minipage}{0.45\textwidth}
    \centering
    \includegraphics[width=\textwidth]{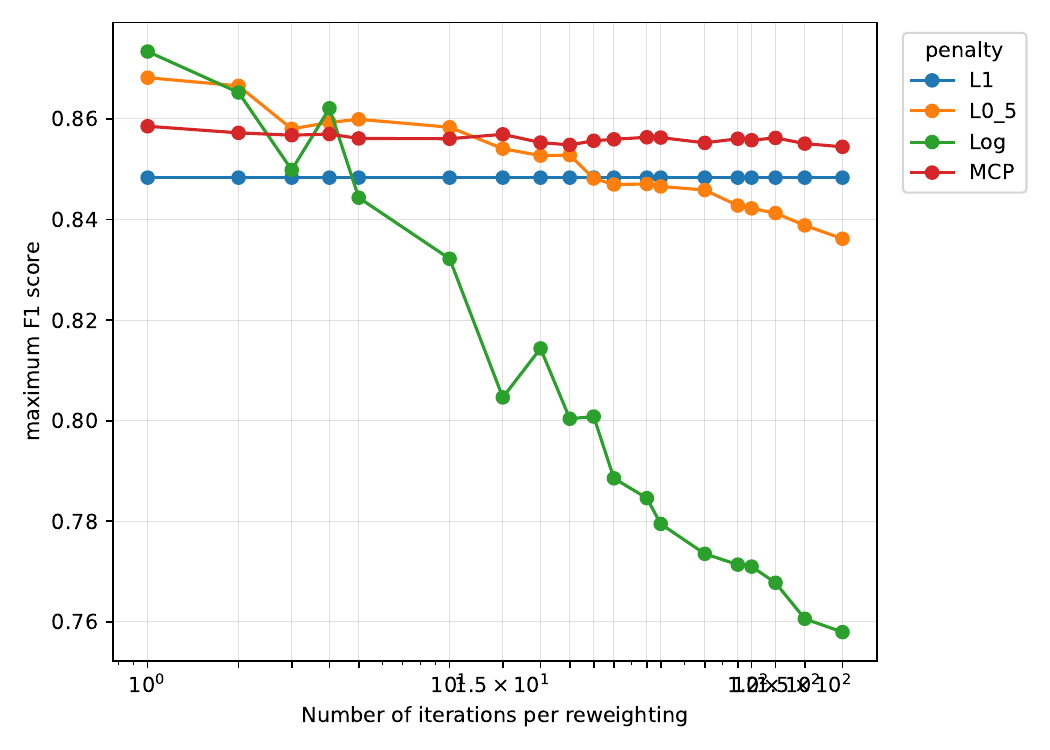}\\
    \includegraphics[width=\textwidth]{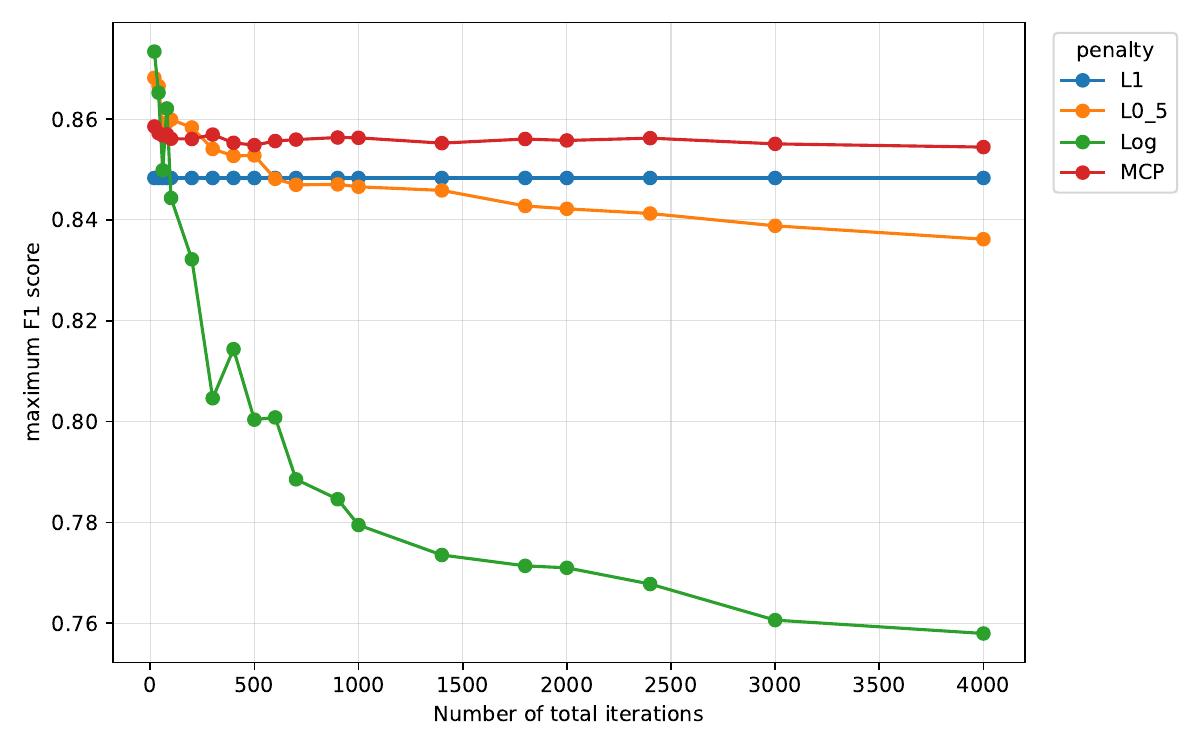}
\end{minipage}\hspace{1em}
\begin{minipage}{0.45\textwidth}
    \centering
    \includegraphics[width=\textwidth]{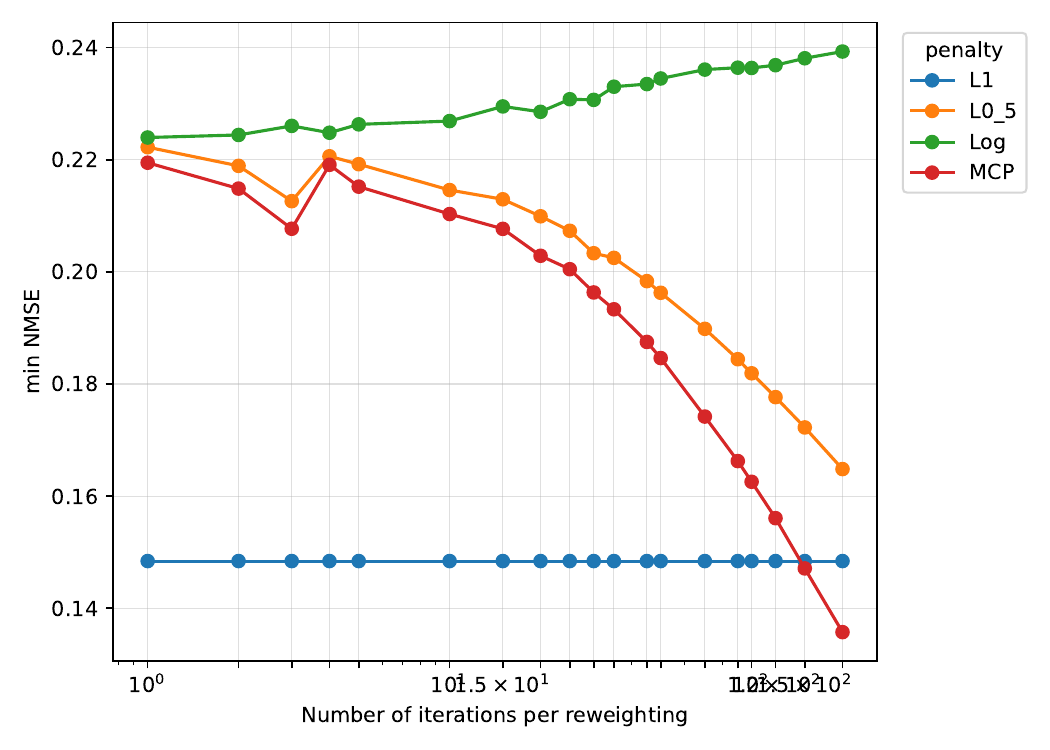}\\
    \includegraphics[width=\textwidth]{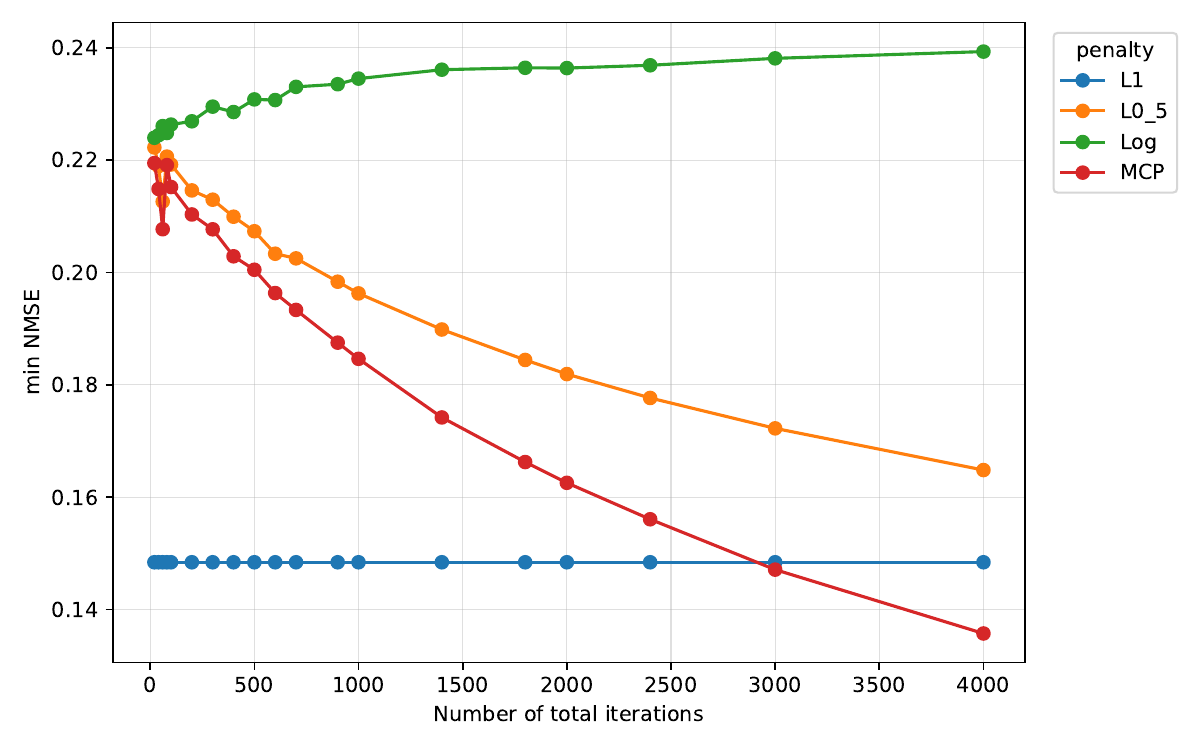}
\end{minipage}
\caption{Graphical-ISTA reconstruction for a sparsity level of $75\%$. Top left: maximum F$1$-score w.r.t. number of iterations between reweightings (log-scale). Bottom left: maximum F$1$-score w.r.t. total number of iterations. Top right: minimum NMSE w.r.t. number of iterations between reweightings (log-scale). Bottom right: minimum NMSE  w.r.t. total number of iterations. } \label{fig:gista0.75} 
\end{figure}

\begin{figure}[t]
\centering
\begin{minipage}{0.45\textwidth}
    \centering
    \includegraphics[width=\textwidth]{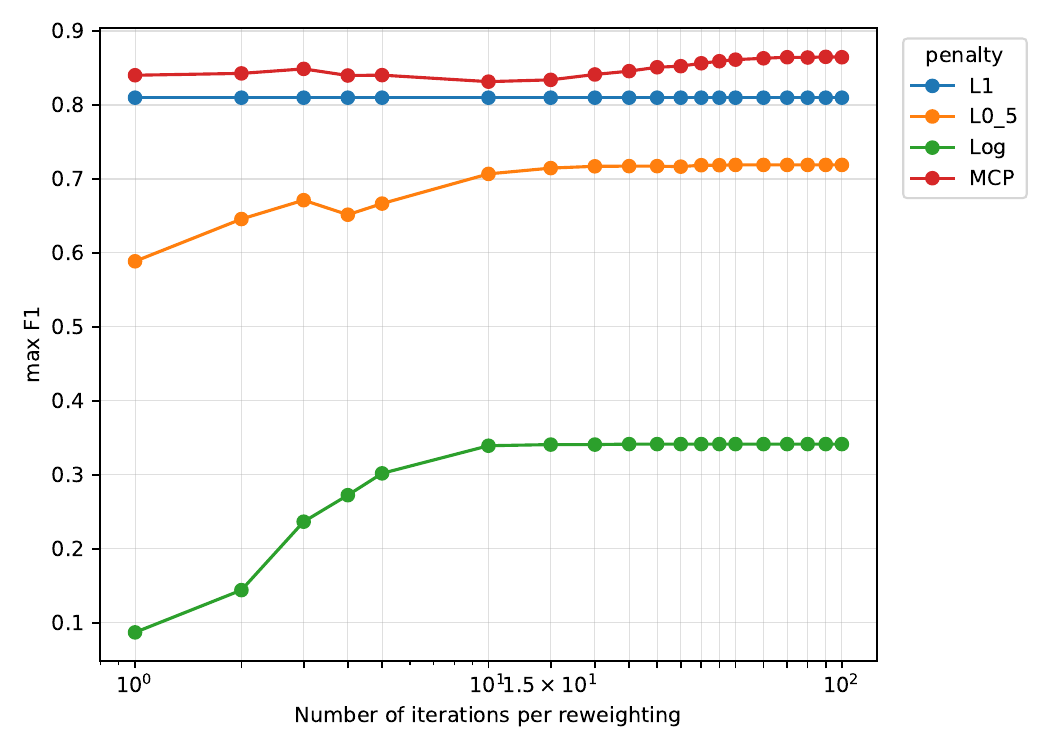}\\
    \includegraphics[width=\textwidth]{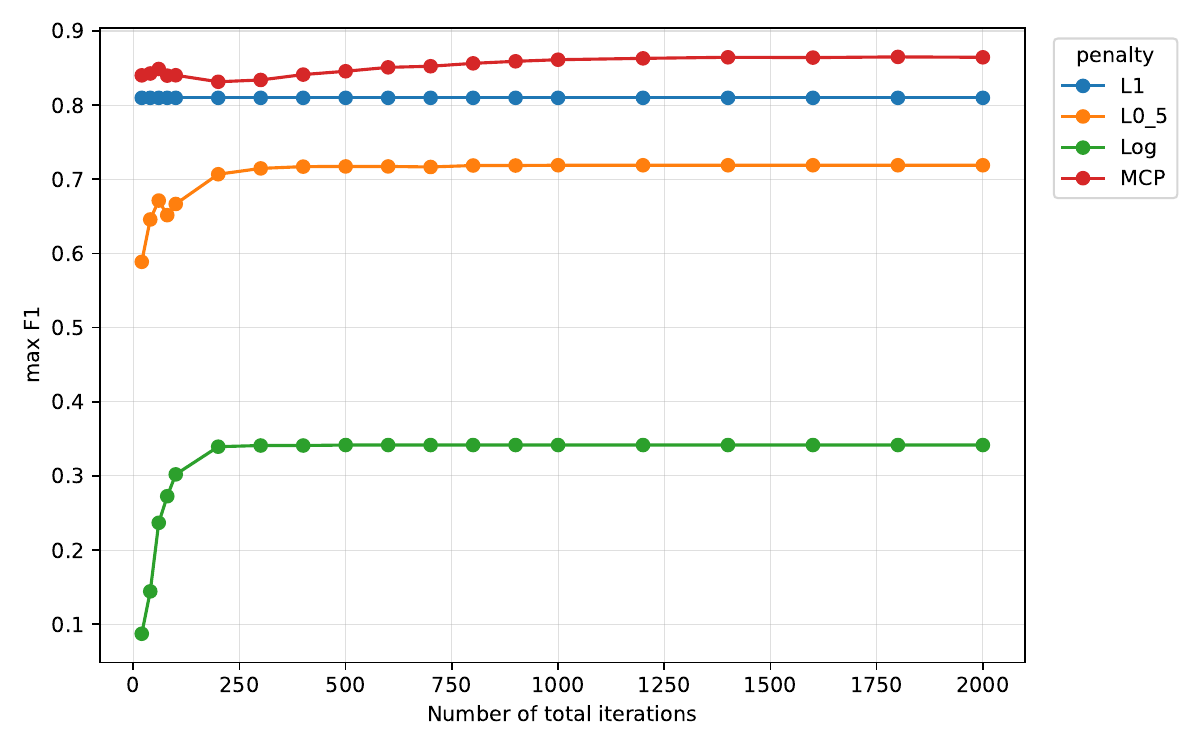}
\end{minipage}\hspace{1em}
\begin{minipage}{0.45\textwidth}
    \centering
    \includegraphics[width=\textwidth]{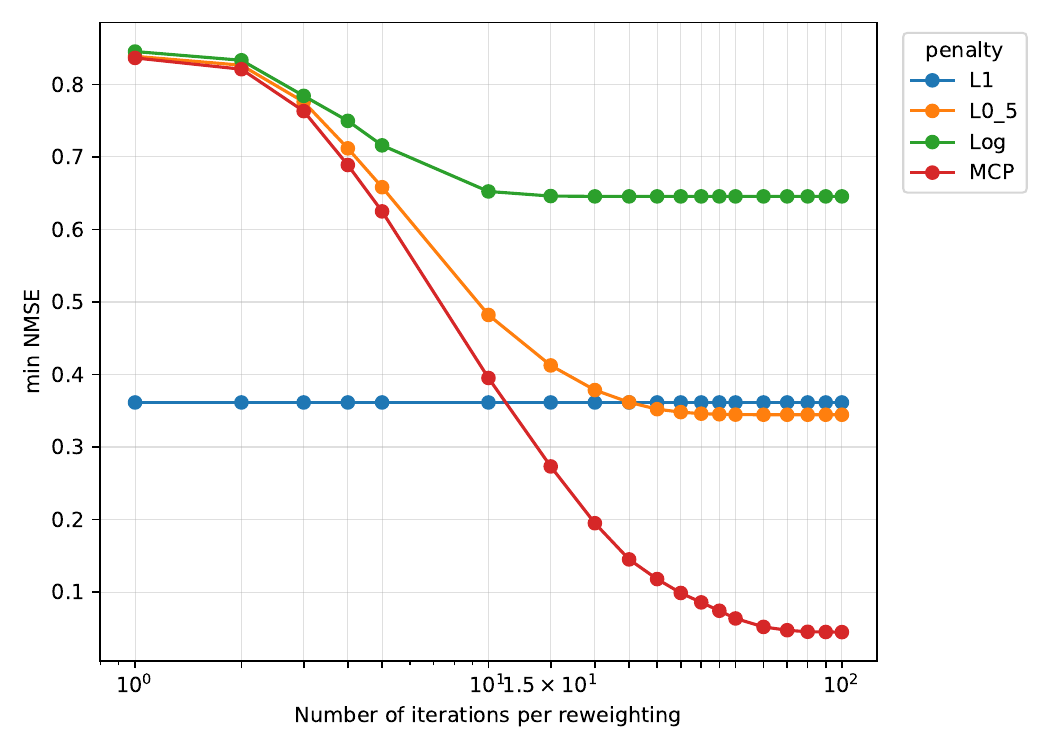}\\
    \includegraphics[width=\textwidth]{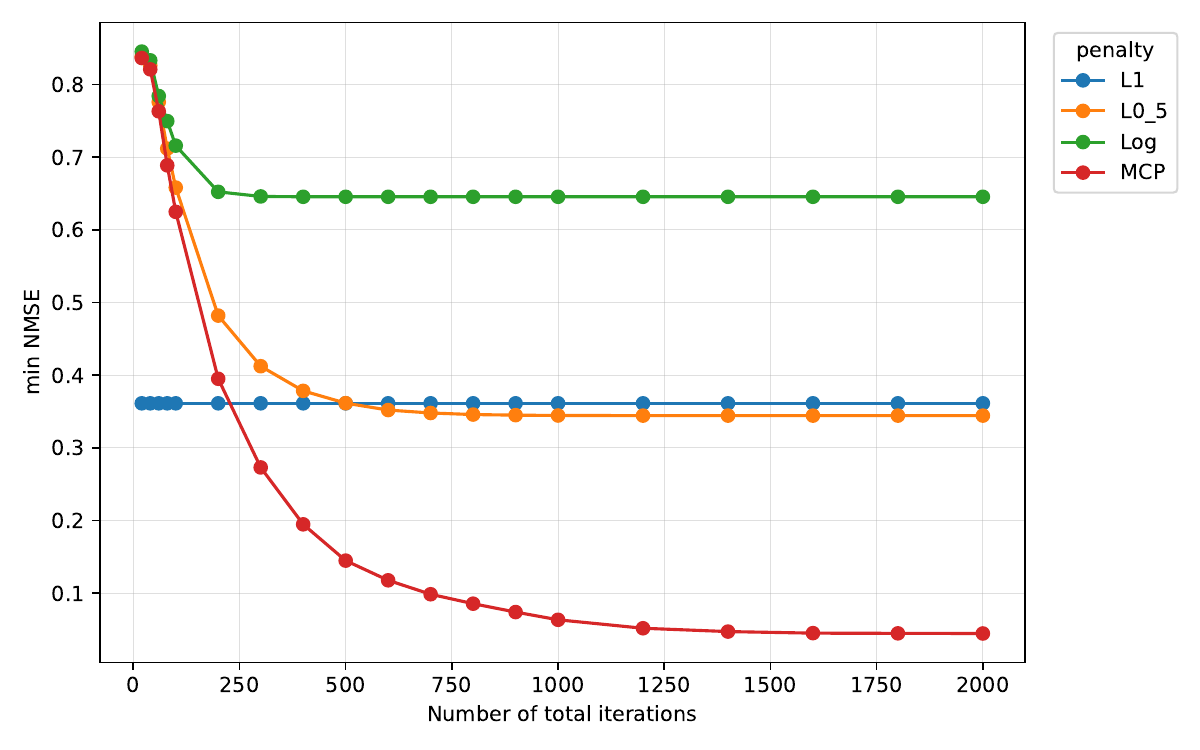}
\end{minipage}
\caption{QUIC reconstruction for a sparsity level of $75\%$. Top left: maximum F$1$-score w.r.t. number of iterations between reweightings (log-scale). Bottom left: maximum F$1$-score w.r.t. total number of iterations. Top right: minimum NMSE w.r.t. number of iterations between reweightings (log-scale). Bottom right: minimum NMSE  w.r.t. total number of iterations.} \label{fig:quic0.75} 
\end{figure}

\begin{figure}[t]
\centering
\begin{minipage}{0.45\textwidth}
    \centering
    \includegraphics[width=\textwidth]{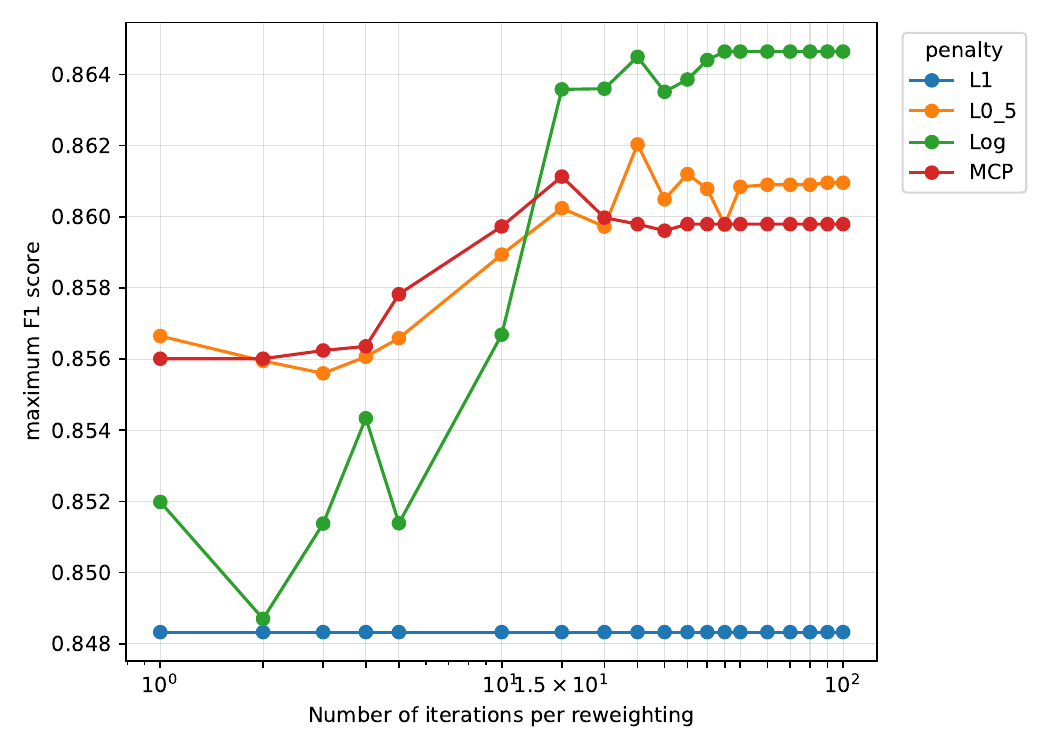}\\
    \includegraphics[width=\textwidth]{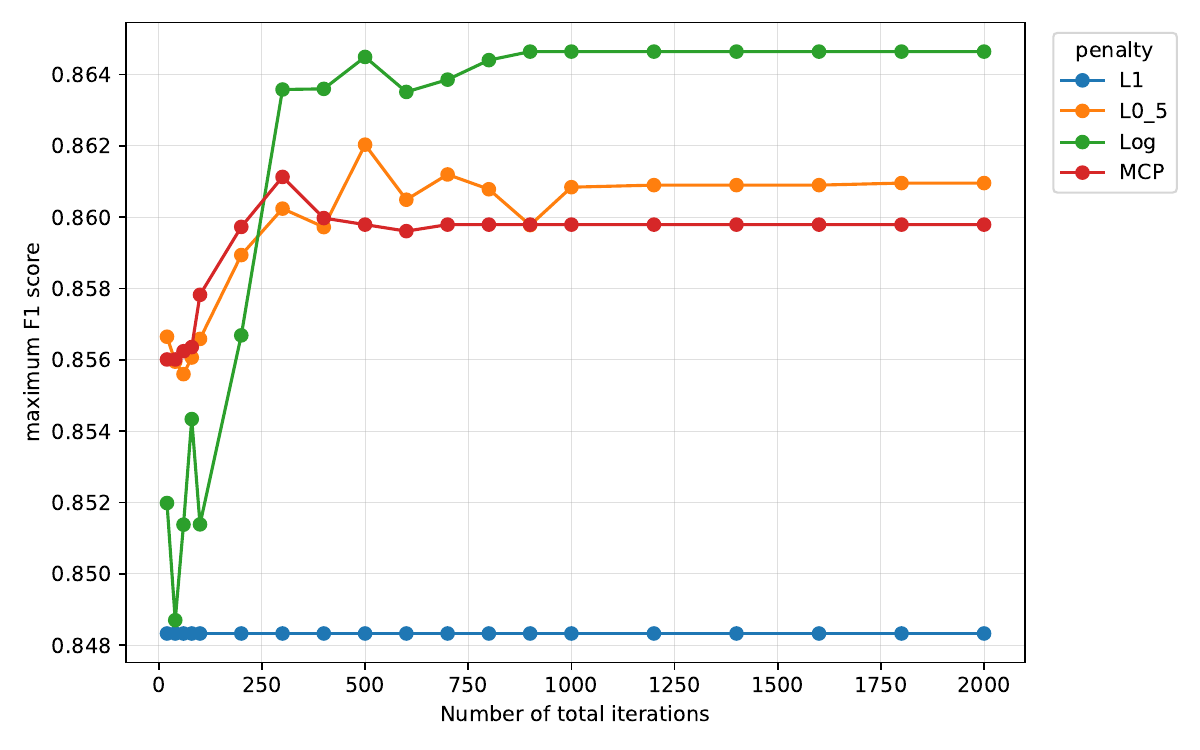}
\end{minipage}\hspace{1em}
\begin{minipage}{0.45\textwidth}
    \centering
    \includegraphics[width=\textwidth]{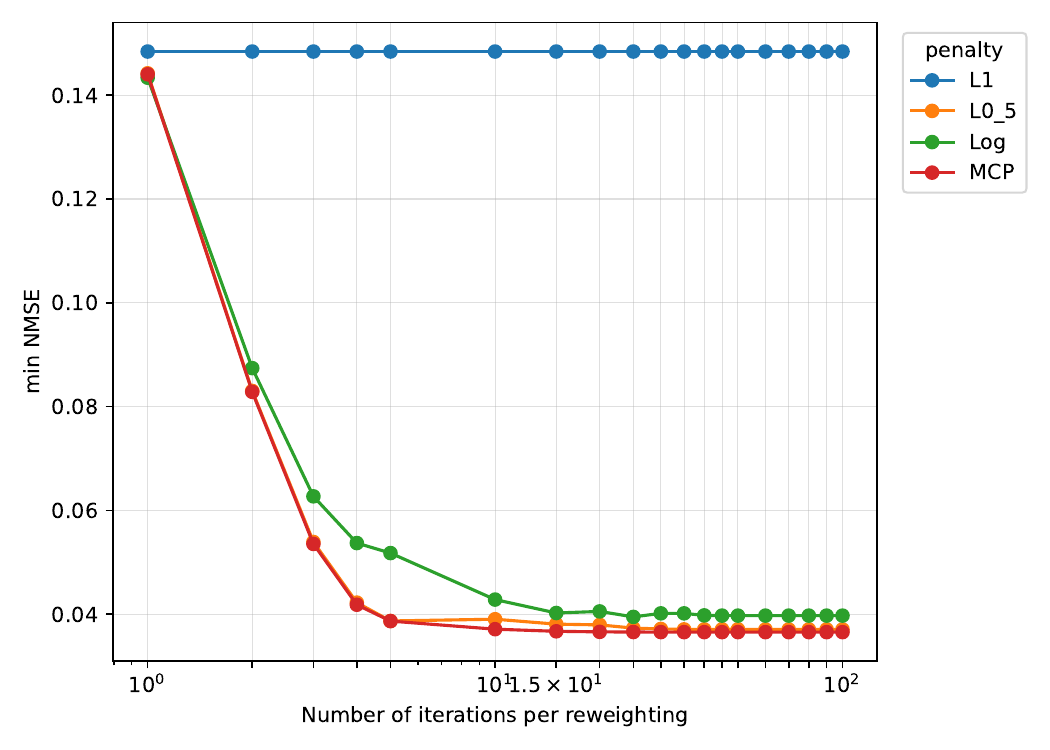}\\
    \includegraphics[width=\textwidth]{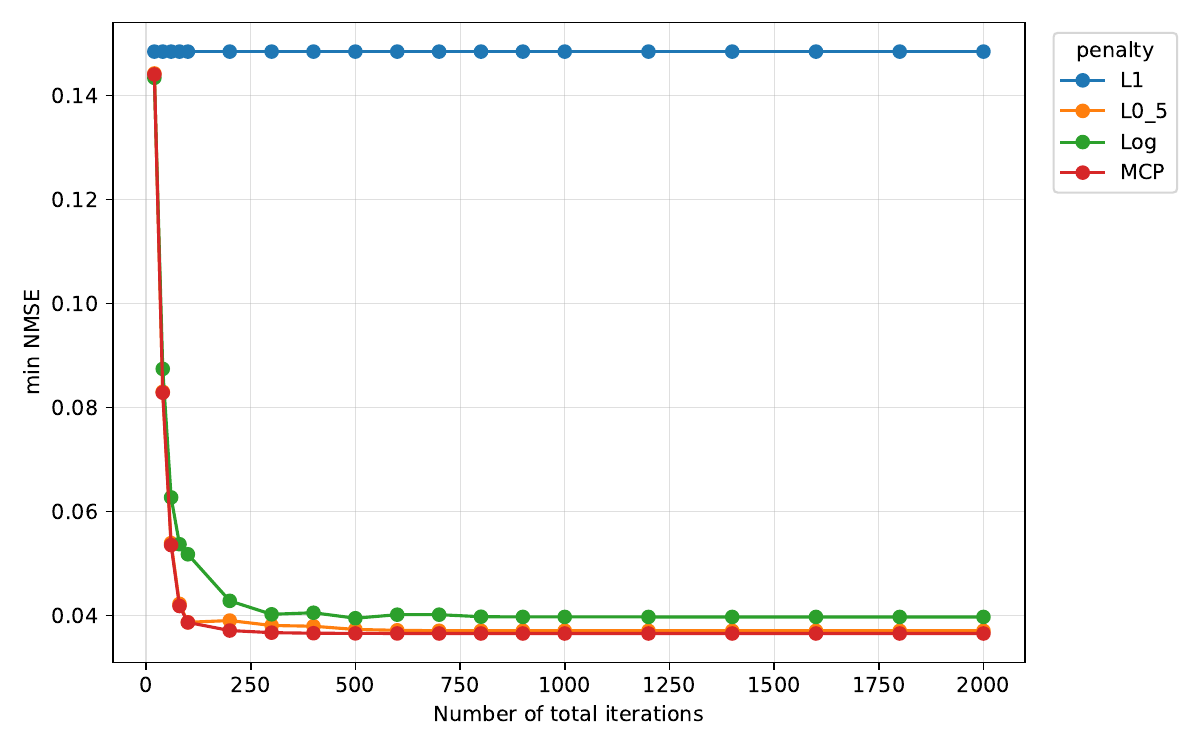}
\end{minipage}
\caption{P-GLasso reconstruction for a sparsity level of $75\%$. Top left: maximum F$1$-score w.r.t. number of iterations between reweightings (log-scale). Bottom left: maximum F$1$-score w.r.t. total number of iterations. Top right: minimum NMSE w.r.t. number of iterations between reweightings (log-scale). Bottom right: minimum NMSE  w.r.t. total number of iterations.} \label{fig:pglasso0.75} 
\end{figure}
\end{appendices}
\end{document}